\documentclass[twoside,11pt]{article}

\usepackage{jmlr2e}

\usepackage{graphicx}
\usepackage{amsmath}
\usepackage{amssymb}
\usepackage{booktabs}
\usepackage{subcaption}

\usepackage{graphicx,amsmath,amssymb,bm,breakurl,epsfig,epsf,color,mathbbol,fmtcount,semtrans,multirow,comment,boldline,pgfplots}
\usepackage{tikz}
\usepackage{url}            
\usepackage{booktabs}       
\usepackage{amsfonts}       
\usepackage{nicefrac}       
\usepackage{microtype}      
\usepackage{mathtools}
\definecolor{darkred}{RGB}{150,0,0}
\definecolor{darkgreen}{RGB}{0,150,0}
\definecolor{darkblue}{RGB}{0,0,200}
\definecolor{purple}{RGB}{100,0,100}

\newtheorem{assumption}{Assumption}

\newtheorem{fact}{Fact}

\newcommand{\BC}{\bar{C}}
\newcommand{\UB}{B}

\newcommand{\SO}[1]{}

\newcommand{\clr}[1]{\textcolor{darkred}{#1}}
\newcommand{\cln}[1]{\textcolor{red}{}}
\newcommand{\som}[1]{\marginpar{\color{darkblue}\tiny\ttfamily SO: #1}}
\newcommand{\ylm}[1]{}

\newenvironment{myenumerate}{
\begin{enumerate}
  \setlength{\topsep}{1mm}
  \setlength{\itemsep}{0.25mm}
  \setlength{\parskip}{0.25mm}
  \setlength{\itemindent}{0mm}
  \setlength{\partopsep}{0mm}
  \setlength{\labelwidth}{15mm}
  \setlength{\leftmargin}{4mm}
}{\end{enumerate}}

\def \endprf{\hfill {\vrule height6pt width6pt depth0pt}\medskip}

\DeclarePairedDelimiter{\floor}{\lfloor}{\rfloor}
\newcommand{\tsn}[1]{{\left\vert\kern-0.25ex\left\vert\kern-0.25ex\left\vert #1 
    \right\vert\kern-0.25ex\right\vert\kern-0.25ex\right\vert}}

\newcommand{\Lci}[1]{{\cal{L}}^{\st,#1}}

\newcommand{\ME}[1]{\text{MM}^{#1}_{\text{emp}}}
\newcommand{\MA}[1]{\overline{\text{MM}}^{#1}_{\text{emp}}}

\newcommand{\frz}{\text{frz}}
\newcommand{\new}{\text{new}}
\newcommand{\lip}[1]{\text{Lip}(#1)}

\newcommand{\eps}{\varepsilon}
\newcommand{\cz}{\eps_0}

\newcommand{\st}{\star}

\newcommand{\NNZ}{\text{NNZ}}
\newcommand{\FLOP}{\text{FLOP}}
\newcommand{\ONNZ}{\overline{\text{NNZ}}}
\newcommand{\OFLOP}{\overline{\text{FLOP}}}

\newcommand{\beq}{\begin{equation}}
\newcommand{\ba}{\begin{align}}
\newcommand{\ea}{\end{align}}

\newcommand{\eeq}{\end{equation}}

\newcommand{\nn}{\nonumber}

\newcommand{\Gb}{{\mtx{G}}}

\newcommand{\Lc}{{\cal{L}}}
\newcommand{\Lco}{{\cal{L}}^{\st}}

\newcommand{\Lch}{{\widehat{\cal{L}}}}

\newcommand{\Lcb}{{\bar{\cal{L}}}}

\newcommand{\Dc}{{\cal{D}}}

\newcommand{\Hb}{{\mtx{H}}}

\newcommand{\seq}{\text{seq}}
\newcommand{\pf}{\phi_{\text{frz}}}
\newcommand{\pfh}{{\phi}_{\text{frz}}}
\newcommand{\pn}{\phi_{\text{new}}}
\newcommand{\pnh}{\hat{\phi}_{\text{new}}}

\newcommand{\order}[1]{{\cal{O}}(#1)}
\newcommand{\ordet}[1]{{\widetilde{\cal{O}}}(#1)}
\newcommand{\ordett}[1]{{\widetilde{\cal{O}}}_{\tau}(#1)}

\newcommand{\z}{{\vct{z}}}

\newcommand{\tn}[1]{\|{#1}\|_{\ell_2}}

\newcommand{\lin}[1]{\|{#1}\|_{L_\infty}}

\newcommand{\tf}[1]{\|{#1}\|_{F}}

\newcommand{\bhi}{\vct{\phi}}

\newcommand{\Sc}{\mathcal{S}}
\newcommand{\Cc}{\mathcal{C}}
\newcommand{\Ccn}{\mathcal{C}^{\bPhi}_{\text{new}}}
\newcommand{\MM}{\text{MM}_{\text{frz}}}
\newcommand{\MS}[1]{\text{MM}_{\text{seq}}^{\st,#1}}

\newcommand{\cc}[1]{\mathcal{C}(#1)}

\newcommand{\Zc}{\mathcal{Z}}
\newcommand{\noi}{\noindent}
\newcommand{\bPne}{{\boldsymbol{\Phi}}_{\text{new},\eps}}

\newcommand{\Rc}{\mathcal{R}}

\newcommand{\bt}{{\boldsymbol{\theta}}}

\newcommand{\bPhi}{{\boldsymbol{\Phi}}}
\newcommand{\bPf}{{\boldsymbol{\Phi}}_{\text{frz}}}
\newcommand{\bPn}{{\boldsymbol{\Phi}}_{\text{new}}}

\newcommand{\bhb}{\vct{\hb}}
\newcommand{\bGam}{{\boldsymbol{\Gamma}}}
\newcommand{\Bc}{\mathcal{B}}

\newcommand{\Mc}{\mathcal{M}}

\newcommand{\vb}{\vct{v}}

\newcommand{\Tn}{T}

\newcommand{\fb}{\vct{f}}

\newcommand{\ab}{\vct{a}}

\newcommand{\hb}{{\vct{h}}}
\newcommand{\hhb}{{\vct{\hat{h}}}}
\newcommand{\hh}{{\hat{h}}}

\newcommand{\Tc}{\mathcal{T}}
\newcommand{\Fc}{\mathcal{F}}

\newcommand{\Xc}{\mathcal{X}}
\newcommand{\Yc}{\mathcal{Y}}

\newcommand{\m}{\vct{m}}
\newcommand{\mn}{\vct{m}^{\text{new}}}
\newcommand{\mnt}{\vct{m}^{\text{new}}_t}
\newcommand{\ma}{\vct{m}^{\text{frz}}}

\newcommand{\bz}{B}
\newcommand{\bo}{\bar{B}}

\newcommand{\x}{\vct{x}}

\newcommand{\y}{\vct{y}}

\newcommand{\W}{\mtx{W}}

\newcommand{\Wc}{{\cal{W}}}

\newcommand{\bgl}{{~\big |~}}

\definecolor{emmanuel}{RGB}{255,127,0}

\newcommand{\R}{\mathbb{R}}
\newcommand{\Pro}{\mathbb{P}}

\newcommand{\E}{\operatorname{\mathbb{E}}}

\newcommand{\vct}[1]{\bm{#1}}
\newcommand{\mtx}[1]{\bm{#1}}

\newcommand{\Vb}{{\mtx{V}}}

\newcommand{\red}{\textcolor{black}}
\newcommand{\blue}{}

\newcommand{\app}{}

\usepackage{hyperref}
\usepackage{url} 
\usepackage{algpseudocode}
\usepackage{algorithm}
\usepackage{wrapfig}

\ShortHeadings{Provable and Efficient Continual Representation Learning}{Li, Li, Asif and Oymak}
\firstpageno{1}

\begin{document}

\title{Provable and Efficient Continual Representation Learning}

\author{\name Yingcong Li \email yli692@ucr.edu \\
       \addr Department of Electrical and Computer Engineering\\
       University of California, Riverside
       \AND
       \name Mingchen Li \email mli176@ucr.edu \\
       \addr Department of Computer Science and Engineering\\
       University of California, Riverside
       \AND
       \name M. Salman Asif \email sasif@ece.ucr.edu \\
       \addr Department of Electrical and Computer Engineering\\
       University of California, Riverside
       \AND
       \name Samet Oymak \email oymak@ece.ucr.edu \\
       \addr Department of Electrical and Computer Engineering\\
       University of California, Riverside
}

\editor{}

\maketitle

\begin{abstract}
In continual learning (CL), the goal is to design models that can learn a sequence of tasks without catastrophic forgetting. Despite rich set of CL techniques, relatively little understanding exists on how representations built by previous tasks benefit new tasks that are added to the network. To this aim, we study \emph{continual representation learning} (CRL) problem where we learn an evolving representation as new tasks arrive. Focusing on zero-forgetting methods where tasks are embedded in subnetworks, we first provide experiments demonstrating CRL can significantly boost sample-efficiency when learning new tasks. To explain this, we establish sample complexity and generalization error bounds for new tasks which formalize the benefits of previously-learned representations. Our analysis and experiments show that CL benefits more when the initial tasks have large sample size and high representation diversity. Diversity ensures that adding new tasks incurs small representation mismatch while requiring few additional trainable weights. Finally, we study computational-efficiency of CRL and propose a novel variant of the PackNet algorithm that employs joint channel \& weight pruning. Our method embeds tasks in channel-sparse subnets requiring up to 80\% less FLOPs to compute while approximately retaining accuracy and is shown to be competitive with various baselines.
\end{abstract} 

\section{Introduction}


\emph{Continual learning} (CL) or lifelong learning aims to build a model for a non-stationary and never-ending sequence of tasks, without access to previous or future data (\cite{thrun1998lifelong,chen2018lifelong,parisi2019continual}). The main challenge in CL is that a standard learning procedure usually results in performance reduction on previously trained tasks if their data is not available. The phenomenon is termed as \emph{catastrophic forgetting} (\cite{mccloskey1989catastrophic,kirkpatrick2017overcoming,pfulb2019comprehensive}).

Importance of continual learning in real-life inference and decision making scenarios led to a rich set of CL techniques~(\cite{delange2021continual,van2019three})\ylm{methods including replay-based, regularization-based, and architecture-based strategies}. However, the theoretical principles of continual learning is relatively less understood and the progress is under-whelming compared to the algorithmic advances despite recent progress {(see \cite{bennani2020generalisation,doan2021theoretical,lee2021continual})}. In this work, for the first time, we investigate the problem of \emph{Continual Representation Learning} (CRL) to answer

\begin{center}
\fbox{\centering\begin{minipage}{1.0\textwidth}
\centering\textit{
What are the statistical benefits of previous feature representations for learning a new task? Can we build an insightful theory explaining empirical performance?}
\end{minipage}}
\end{center}


Our key contribution is addressing both questions affirmatively. We develop our algorithm and theory for architecture-based zero-forgetting CL which includes  PackNet~(\cite{mallya2018packnet}), CPG~(\cite{hung2019compacting}), and RMN~(\cite{wu2020understanding}). These methods eliminate forgetting by training a sub-network for each task and freezing the trained parameters. At a high level, all of these methods inherently have the potential of continual representation learning by allowing new tasks to reuse the frozen feature representations built for earlier tasks. However, quantifying the empirical benefits/performance of these representations and building the associated theory have been elusive. We overcome this via innovations in experiment design, theory, and algorithms:

\begin{figure}[t]
\vspace{-20pt}
\centering
\begin{minipage}[t]{.65\textwidth}
  \centering
  \begin{tikzpicture}
  \hspace{-12pt}
        \node at (0,0) [scale=1.1] {\includegraphics[width=\linewidth]{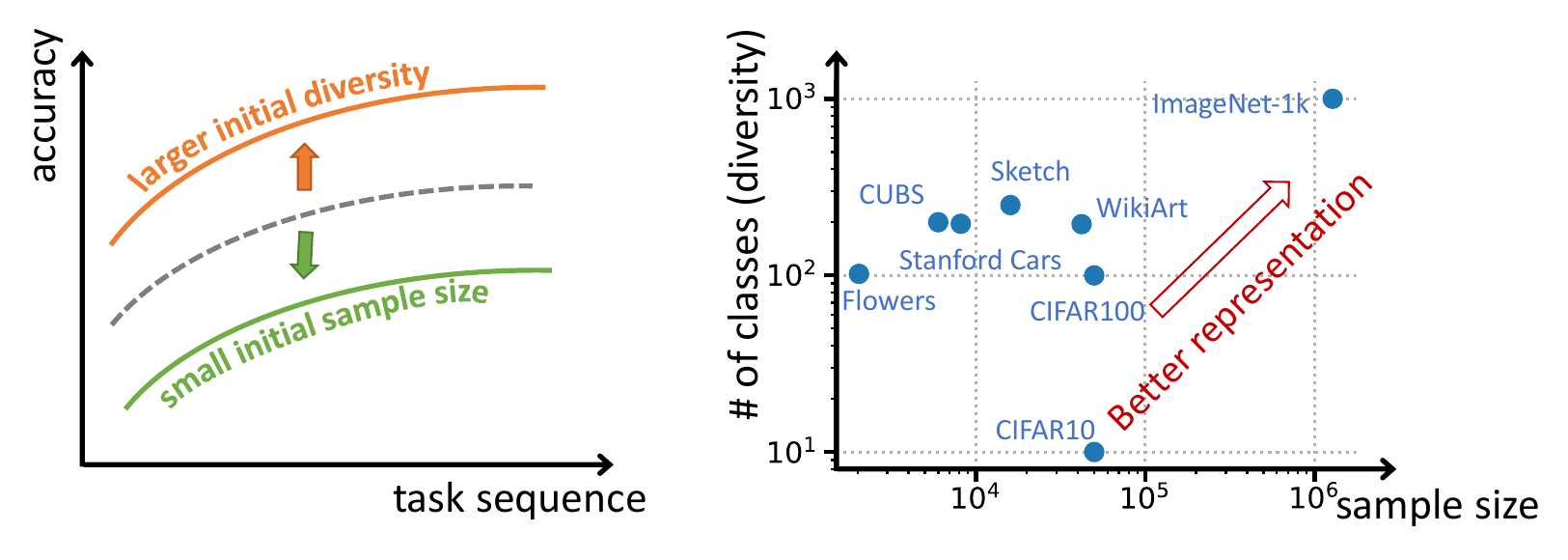}};
        \node at (-3.2,-1.8) [scale=0.8]{(a)};
        \node at (2.3, -1.8) [scale=0.8]{(b)};
    \end{tikzpicture}\vspace{0pt}

\end{minipage}
\caption{{We show that task diversity and sample size are two essential factors in continual representation learning. In (a), dashed black curve depicts the test accuracy when CL tasks arrive in random order. The orange and green curves depict the test accuracies when tasks are ordered with decreasing diversity (more diverse first) and increasing sample size (small sample first) respectively. Figure (b) displays the \# of classes (a natural proxy for diversity) and sample size for various datasets. 
	}}	\label{fig:intro}\vspace{-7pt}

\end{figure}

\noindent$\bullet$ \textbf{Theoretical and empirical benefits of continual representations.} We establish theoretical results and sample complexity bounds for CRL by using tools from empirical process theory. Within our model, a new task uses frozen feature map $\pf$ of previous tasks and learns an additional task-specific representation $\pn$. For PackNet (\cite{mallya2018packnet}), $\pf$ corresponds to all the nonzero weights so far and $\pn$ is the nonzeros allocated to the new task, thus $\pf$ requires a lot more data to learn.  Our theory (see Section~\ref{sec:crl_theory}) explains (1) how the new task can reuse the frozen feature map to greatly reduce the sample size and (2) how to quantify the representational compatibility between old and new tasks. Specifically, we fully avoid the statistical cost of learning $\pf$ and replace it with a \emph{``representational mismatch''} term between the new task and frozen features. \red{Additionally, inspired by \cite{srebro2010smoothness,bartlett2005local}, we obtain faster rates using a localization argument for our CRL problem (see Section~\ref{sec:crl_theory}).} \app{We also extend our results from a single task to learning a sequence of tasks by quantifying the aggregate impact of finite data on representation $\pf$ which evolves as new tasks arrive (see Section~\ref{app C})}.

An important conclusion is that the ideal frozen representation $\pf$ should contain \textbf{diverse and high-quality features} so that it has small representational \emph{mismatch} with new tasks. {This conclusion is depicted in Figure~\ref{fig:intro} where we highlight the impact of diversity and sample size on task accuracy.} It is also in line with the intuitions from multitask representation learning (\cite{maurer2016benefit,tripuraneni2020theory}) and motivates the following CL principle supported by our experiments:
\begin{quote}
\hspace{-27pt}\boxed{\textit{First learn diverse and large-data tasks so that their representations help upcoming tasks.}}
\end{quote}

 Indeed, we show in Fig.~\ref{fig:LS} that a new task with small data achieves significantly higher accuracy if it is added later in the task sequence (so that $\pf$ becomes more diverse). Then, we show in Fig.~\ref{fig:diversity} that choosing the first task to be diverse (such as ImageNet) helps all the downstream tasks. Finally, we show in Fig.~\ref{fig:Sample} that it is better to first learn tasks with large sample sizes to ensure high-quality features. 
{Our results on the importance of task order also relate to curriculum learning (\cite{bengio2009curriculum}) where the agent gets to choose the order tasks are learned. However, instead of curriculum learning, which aims to learn one task from easy to hard, our work is based on a more general continual learning setup. Our conclusion on task diversity provides a new perspective on the learning order of curriculum learning that we should first learn more representative tasks instead of the easier tasks. }

\noindent$\bullet$ \textbf{Algorithms for inference-efficient CRL.} Zero-forgetting CL methods often need a large neural network (dubbed as supernetwork) to load numerous tasks into subnetworks, or dynamically expand the model to avoid forgetting (\cite{delange2021continual}). Thus, CL/CRL may incur a large computational cost during inference. This leads us to ask whether one can retain the accuracy benefits of CRL while ensuring that each new task utilizes an \textbf{inference-efficient} sub-network, thus achieving the best of both worlds. We quantify inference-efficiency via floating point operations (FLOPs) required to compute the task subnetwork. To this end, we propose Efficient Sparse PackNet (ESPN) algorithm (Fig.~\ref{overview fig}). In a nutshell, ESPN guarantees inference-efficiency by incorporating a channel-pruning stage within PackNet-style approaches. Via extensive evaluations, we find that, ESPN incurs minimal loss of accuracy while greatly reducing FLOPs (as much as 80\% in our SplitCIFAR-100 experiments, Table \ref{CIFARtable}).

{In summary, this work makes key contributions to continual representation learning from empirical, theoretical, and algorithmic perspectives. {We remark that our main work focuses on the task-incremental CL setting with parameter-isolation.} However, we anticipate that the high-level principles arising from our findings will inform alternative continual learning strategies as well. In the remainder of the paper, we discuss the related work, detail our empirical and theoretical findings on CRL, and present the ESPN algorithm and its evaluations. }
\section{Related Work}\label{app:related}


Our contributions are most closely related to the continual learning literature. Our theory and algorithm also connect to representation learning and neural network pruning.

\noindent\textbf{Continual learning.}
A number of methods for continual and lifelong learning have been proposed to tackle the problem of catastrophic forgetting and existing approaches can be broadly categorized into three groups~(\cite{delange2021continual}): replay-based~(\cite{lopez2017gradient,rebuffi2017icarl,rolnick2018experience,buzzega2021rethinking,aljundi2019gradient,borsos2020coresets}), regularization-based~(\cite{kirkpatrick2017overcoming,zenke2017continual,li2017learning,jung2020continual}), and architecture-based methods~(\cite{fernando2017pathnet,yoon2017lifelong,mallya2018piggyback,rusu2016progressive}). \red{\cite{lubana2021quadratic} shows that the quadratic regularization technique plays a rule in preventing catastrophic forgetting.} Recent work \cite{ramesh2021model} explores statistical challenges associated with continual learning in terms of the relatedness across tasks through replay-based strategies. In comparison, we focus on the benefits of representation learning and establish how representation built for previous tasks can drastically reduce the sample complexity on new tasks in terms of representation mismatch without access to past data. Another key difference is that our analysis allows for learning multiple sequential tasks rather than a single task. In our work, consistent with our theory, we focus on zero-forgetting CL (\cite{mallya2018packnet,hung2019compacting,mallya2018piggyback,wortsman2020supermasks,kaushik2021understanding,tu2020extending}), which is a sub-branch of architecture-based methods and completely eliminates forgetting. PackNet~(\cite{mallya2018packnet}), CPG~(\cite{hung2019compacting}), and RMN~(\cite{kaushik2021understanding}) train a sub-network for each task and implement zero-forgetting by freezing the trained parameters, while Piggyback (\cite{mallya2018piggyback}) trains masks only over pretrained model. Inspired by \cite{ramanujan2020s}, SupSup (\cite{wortsman2020supermasks}) trains a binary mask for each task while keeping the underlying model fixed at initialization. However, finding of \cite{frankle2018lottery} shows that a network can reduce by over 90\% of parameters without performance reduction. This inspires our weight-allocation strategy to adapt PackNet to more sparse sub-networks. Unlike PackNet which prunes the network by keeping the largest absolute weights in each layer and reuses all the frozen weights, CPG and RMN apply real-valued mask to each fixed entry and prune by keeping the largest values of the mask. SupSup is motivated by \cite{ramanujan2020s,zhou2019deconstructing,malach2020proving} which show that a sufficiently over-parameterized random network contains a sub-network with roughly the same accuracy as the target network without training. In essence, it aims to find masks only over random network, and it is adaptable to infinite tasks. However, it leads to inefficient inference (due to using the full network) and potentially large memory requirements (as one has to store a mask as large as supernet rather than a subnet).

In CL, in order to embed the super-network with many tasks or to achieve acceptable performance over a masked random network, sufficiently large networks are needed, which naturally leads to inefficiency during inference-time without proper safeguards. Addressing this challenge appears to be an unexplored avenue as far as we are aware. While in this work, we present Efficient Sparse PackNet (ESPN) that implements zero-forgetting CL and achieves state-of-the-art accuracy with less computational demand.

We also emphasize that there are several interesting works on the theory of continual learning such as \cite{lee2021continual, doan2021theoretical,bennani2020generalisation,mirzadeh2020understanding,yin2020optimization}. These works focus on NTK-based analysis for deep nets, theoretical investigation of orthogonal gradient descent (\cite{farajtabar2020orthogonal}), and task-similarity. However, to the best of our knowledge, ours is the first work on the representation learning ability and the associated data-efficiency.


\noindent\textbf{Representation learning theory.} The rise of deep learning motivated a growing interest in theoretical principles behind representation learning. Similar in spirit to this project, \cite{maurer2016benefit} provides generalization bounds for representation-based transfer learning in terms of the Rademacher complexities associated with the source and target tasks. Some of the earliest works towards this goal include \cite{baxter2000model}~and linear settings of \cite{lounici2011oracle,pontil2013excess,wang2016distributed,cavallanti2010linear}. More recent works (\cite{hanneke2020no,lu2021power,kong2020meta,qin2022non,wu2020understanding,garg2020functional,gulluk2021sample,xu2021statistical,du2020few,tripuraneni2020theory,qin2022non,tripuraneni2020provable,sun2021towards,maurer2016benefit,arora2019theoretical,chen2021weighted,sun2021towards}) consider variations beyond supervised learning, concrete settings or established more refined upper/lower performance bounds. There is also a long line of works related to model agnostic meta-learning (\cite{finn2017model,denevi2019online,balcan2019provable,khodak2019adaptive}). {Unlike these works, we consider the CL setting where tasks arrive sequentially and establish how the representation learned by earlier tasks provably helps learning the new tasks with fewer samples and better accuracy.} \red{Similar to \cite{bartlett2005local,srebro2010smoothness}, we also obtain faster rate bounds using a localization argument but for continual learning problem with multiple tasks. To the best of our knowledge, existing works on multitask learning do not achieve fast rates except for linear representations (\cite{tripuraneni2020provable,sun2021towards}) which is a distinguishing feature of our bound.}


\noindent\textbf{Neural network pruning.} Our work is naturally related to neural network pruning methods and compression techniques as we embed tasks into sub-networks. Large model sizes in deep learning have led to a substantial interest in model pruning/quantization (\cite{han2015deep,hassibi1993second,lecun1990optimal}). DNN pruning has a diverse literature with various architectural, algorithmic, and hardware considerations (\cite{sze2017efficient,han2015learning}).  Here, we mention the ones related to our work. \cite{frankle2018lottery} empirically shows that a large DNN contains a small subset of favorable weights (for pruning), which can achieve similar performance to the original network when trained with the same initialization. \cite{zhou2019deconstructing,malach2020proving,pensia2020optimal} demonstrate that there are subsets with good test performance even without any training and provide theoretical guarantees. In relation, \cite{chang2021provable} establishes the theoretical benefits of training large over-parameterized networks to improve downstream pruning performance.

Although weight pruning is proven to be a good way to reduce model parameters and maintain performance, practically, it does not lead to compute efficiency except for some dedicated hardwares~(\cite{han2016eie}). Unlike weight pruning, structured/channel pruning prunes the model at the channel level which results in a slim sub-network carrying much fewer FLOPs than the original dense model~(\cite{liu2017learning, zhuang2020neuron, wen2016learning, ye2018rethinking}). For example~\cite{wen2016learning,zhou2016less,alvarez2016learning,lebedev2016fast,he2017channel} prune models by adding a sparse regularization over model weights whereas \cite{liu2017learning, zhuang2020neuron} only add regularization over channel factors, and prune channels with lower scaling factors. However, these prior works do not focus on the scenario where almost all of FLOPs pruned, for example with only 1\% of original FLOPs remained. To achieve this goal, we present an innovative channel pruning method based on FLOP-aware penalization. Our technique is inspired from \cite{liu2017learning, zhuang2020neuron} (they use sparsity regularization over BatchNorm weights only). However, as demonstrated {in Appendix~\ref{pruning sec}}, our pruning method outperforms both of them in our evaluations.
\begin{figure*}[t]
\vspace{-7pt}
    \centering
    \begin{tikzpicture}
    \node at (0,0) {
    \includegraphics[width=1\textwidth]{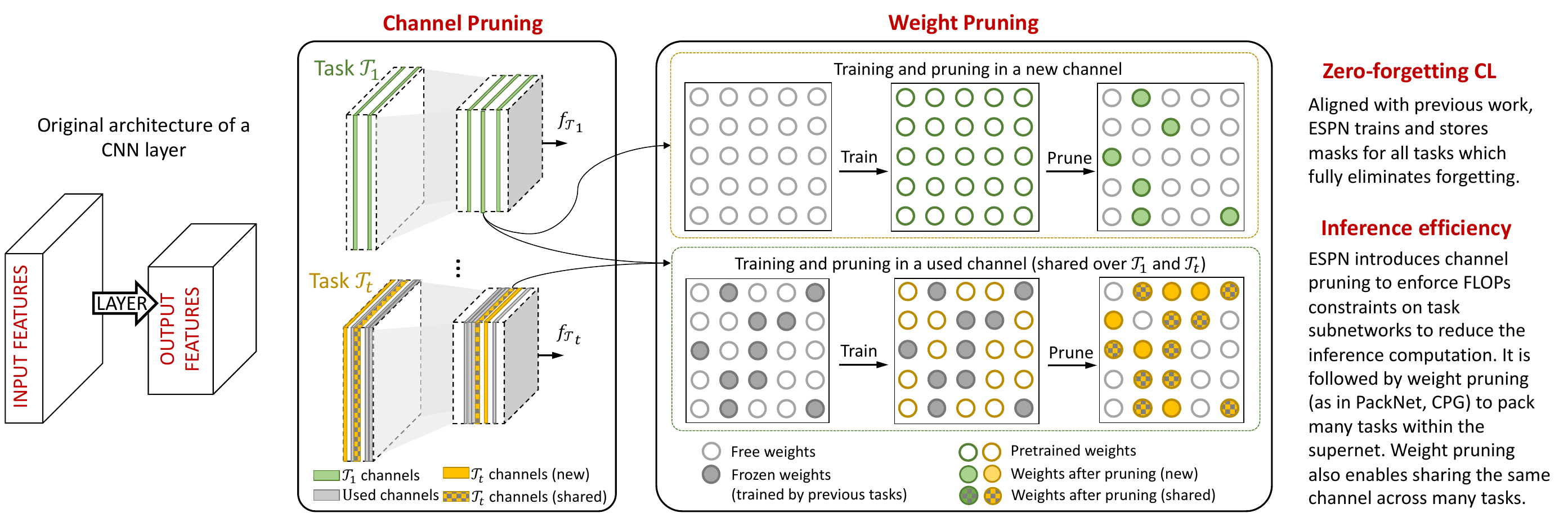}
    };
    \end{tikzpicture}\vspace{-7pt}
    \caption{This figure summarizes our ESPN algorithm that infuses inference efficiency in PackNet-style CL techniques by jointly pruning channels and weights. \textbf{Left side} displays the input and output feature tensors of a convolutional layer of the original supernet architecture. \textbf{Right side} depicts adding two tasks into the network.  \textbf{{Channel pruning}} over the input and output features of $\Tc_{1}$ to $\Tc_t$ is used to enforce FLOPs constrained by selecting a few channels out of the total channels (Green and Yellow planes in the network). We remark that used filters can be selected again and shared by multiple tasks to enable representation sharing. 
    \textbf{Weight Pruning} depicts the evolution of the sub-filters within channels as we add more tasks. Weight pruning within sub-networks enables the sub-filters to be shared across tasks. For each task, we firstly train all free weights on the task data (Green/Yellow rings). After training, we use magnitude pruning to sparsify sub-filters on the entire filter including frozen weights (Gray circles). For $\Tc_t$, since some channels are shared with previous tasks, the Gray circles are frozen and only the Yellow rings are trained. However during pruning, we consider all weights and pick both solid Yellow circles (new weights) from free weights and Yellow mosaic circles (used/shared weights) from frozen weights.}
    \label{overview fig}\vspace{-7pt}
\end{figure*}
\section{Efficient Sparse PackNet (ESPN) Algorithm}\label{sec: espnalgo}
 We will use PackNet and ESPN throughout the paper to study continual representation learning. Thus, we first introduce the high-level idea of our ESPN algorithm which augments PackNet. Suppose we have a single model referred to as supernetwork and a sequence of tasks. Our goal is to train and find optimal sparse sub-networks within the supernet that satisfy both FLOPs and sparsity restrictions without any performance reduction or forgetting of earlier tasks. Figure~\ref{overview fig} illustrates our proposed algorithm. We propose a joint channel and weight pruning strategy in which FLOPs constraint (in channel pruning) is important for efficient inference and sparsity constraint (in weight pruning) is important for \emph{packing} all tasks into the network even for a large number of tasks $T$. In essence, ESPN equips PackNet-type methods with inference-efficiency using an innovative FLOP-aware channel pruning strategy. \app{Section \ref{sec: espn_detail} and appendix provide implementation details and inference-time evaluations on ESPN.}

\noindent{\textbf{Notation.}} ESPN admits FLOPs constraint as an input parameter, which is a critical aspect of the experimental evaluations. Let MAX\_FLOPs be the FLOPs required for one forward propagation through the dense supernetwork. In our evaluations, for $\gamma\in (0,1]$, we use \textbf{ESPN-$\gamma$} to denote our CL Algorithm \ref{algo 1} in which each task obeys the FLOPs constraint $\gamma\times \text{MAX\_FLOPs}$. Similarly, \textbf{Individial-$\gamma$} will be the baseline that each task is trained individually (from scratch) on the full supernetwork while using at most $\gamma\times \text{MAX\_FLOPs}$. \textbf{Individual} is same as \textbf{Individual}-1 where we train the whole network without pruning.


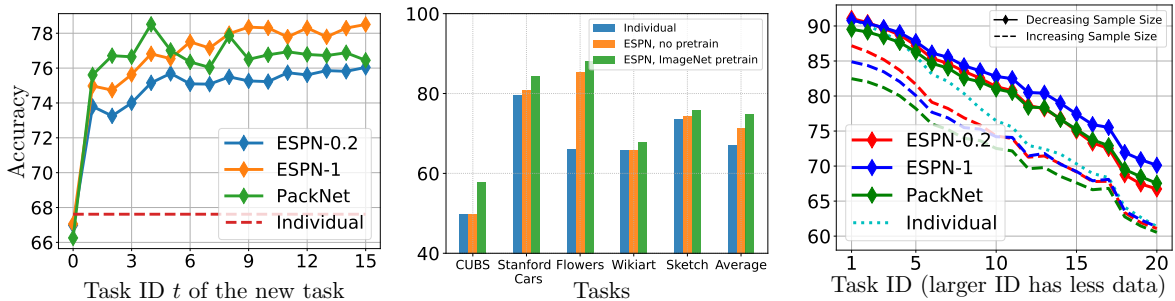
\begin{figure*}[t]
\vspace{-7pt}
\centering
\begin{subfigure}[t]{.315\textwidth}
  \centering
  \begin{tikzpicture}\hspace{-10pt}
        \node at (0,0) [scale=1.15] {\includegraphics[width=\linewidth]{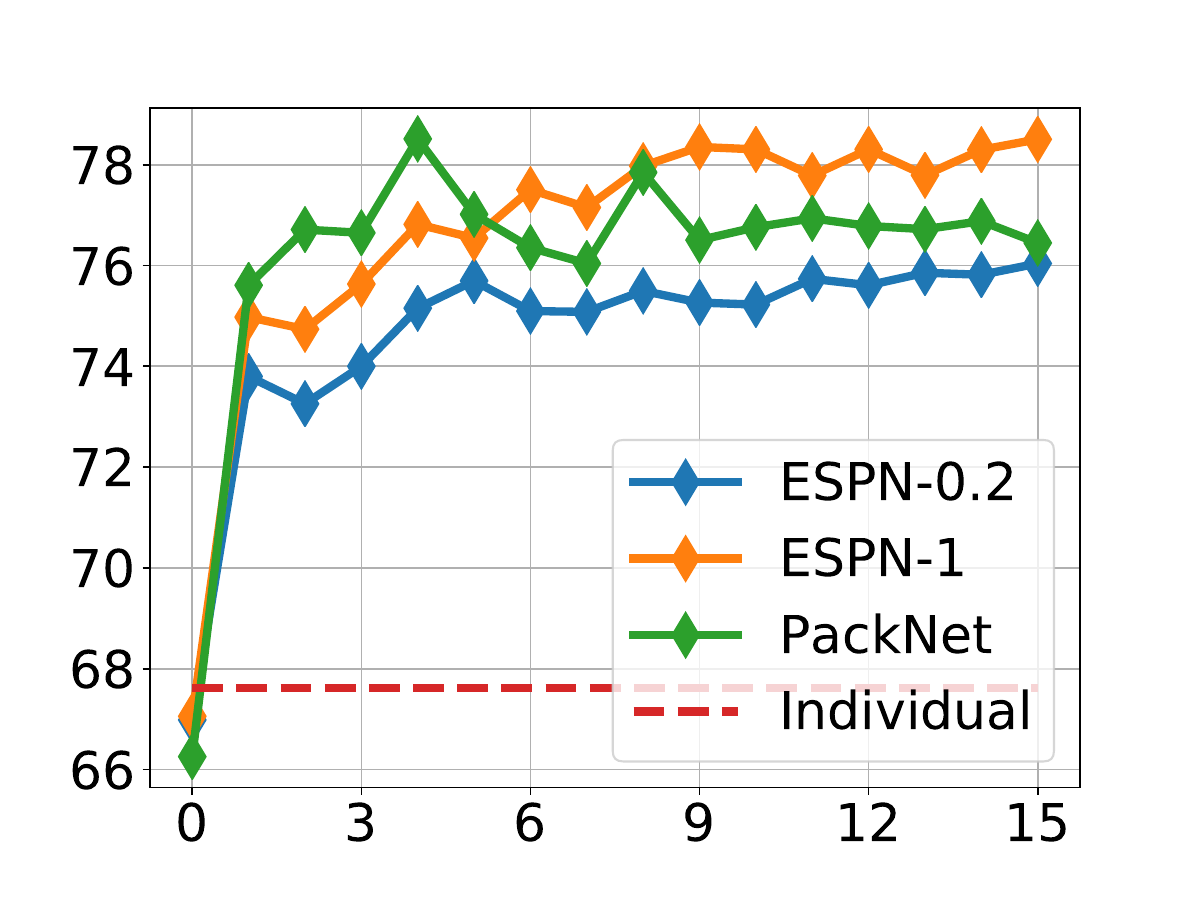}};
        \node at (0,-2.1) [scale=0.8] {{Task ID $t$ of the new task}};
        \node at (-2.6,0) [scale=0.8,rotate=90] {Accuracy};
    \end{tikzpicture}
        \centering
  \caption{Learning the small task later down the line helps noticeably. Here, we first learn $t-1$ SplitCIFAR100 tasks continually with full datasets and add task $t$ with a small dataset.}
	\hspace{-10pt}\label{fig:LS}
\end{subfigure}\hspace{10pt}\begin{subfigure}[t]{.315\textwidth}
  \centering
  \begin{tikzpicture}\hspace{-10pt}
        \node at (-0,0) [scale=1.15] {\includegraphics[width=\linewidth]{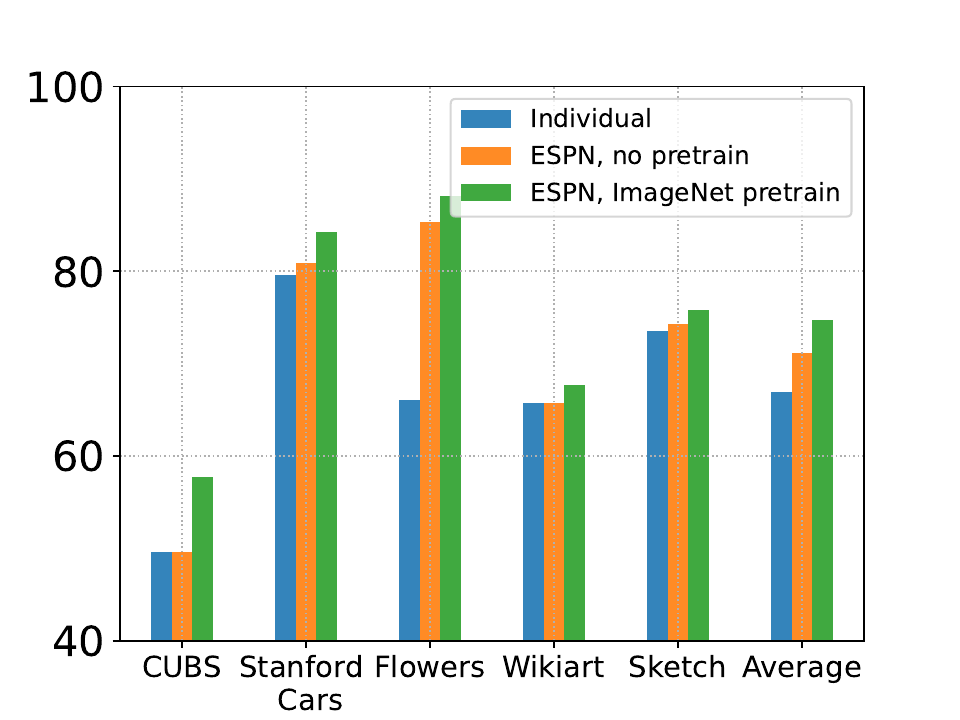}};
        \node at (0,-2.1) [scale=0.8] {Tasks};
    \end{tikzpicture}
    \caption{{Very diverse initial tasks boost all downstream tasks. Here, Green bars continually learn (CUBS to Sketch) with ImageNet as first task, whereas Orange bars are without.
    }}\label{fig:diversity}
	
\end{subfigure}\hspace{10pt}\begin{subfigure}[t]{.315\textwidth}
  \centering
   \begin{tikzpicture}\hspace{-10pt}
        \node at (0,-0.) [scale=1.2] {\includegraphics[width=\linewidth]{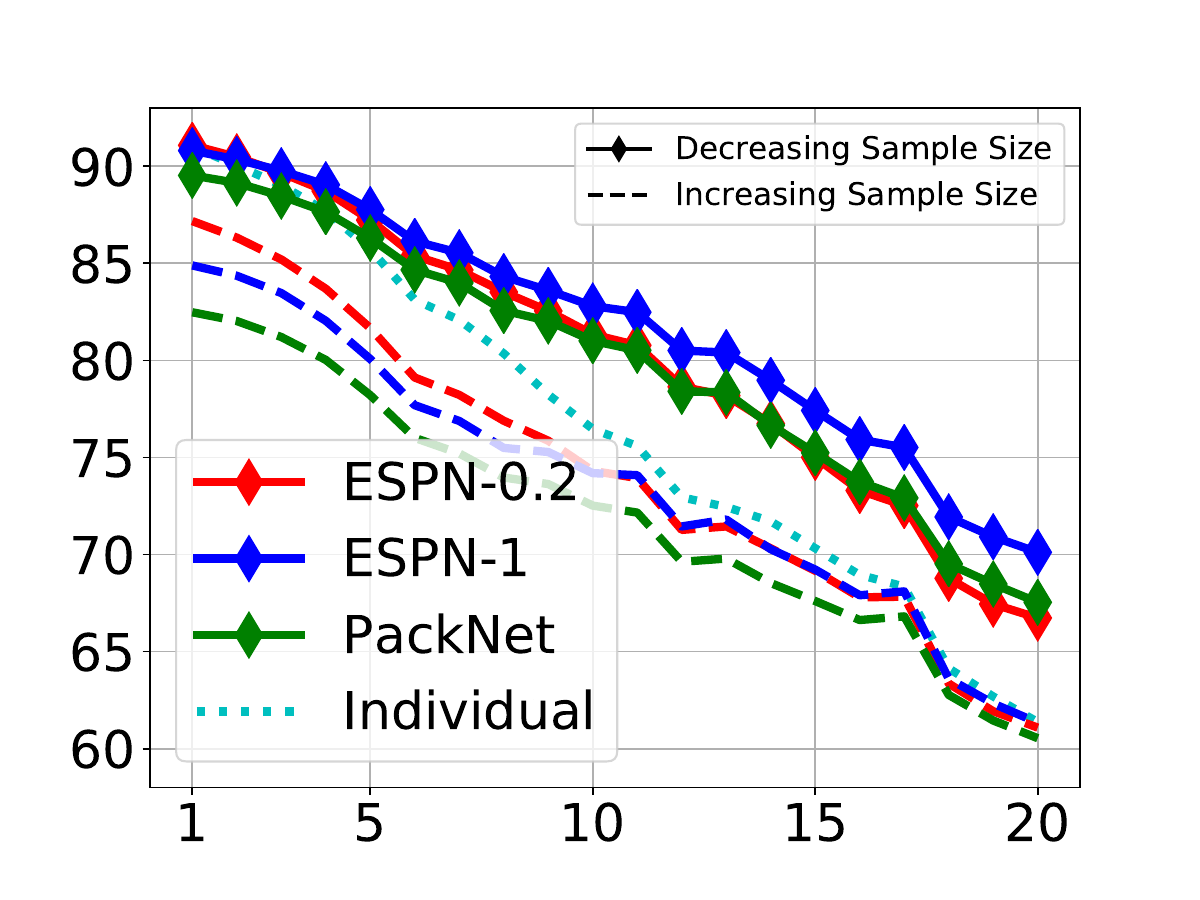}};
        \node at (0,-2.1) [scale=0.8] {Task ID (larger ID has less data)};
    \end{tikzpicture}
    \centering
    \caption{{Learning large sample tasks first enables building better representations for future small tasks. Here, solid curves train large tasks first whereas dashed curves train small tasks first.
    }}\label{fig:Sample}
	
\end{subfigure}\vspace{-10pt}
\caption{{These figures show motivating empirical findings on how representations of earlier tasks help new tasks in continual learning (see Sec.~\ref{sec:crl_exp} for further discussion). ESPN-$\gamma$ is our inference-efficient algorithm (Sec.~\ref{sec: espnalgo}) and with $\gamma=0.2$, each task requires only 20\% compute during inference. In Section \ref{sec:crl_theory} we develop theoretical analysis and provable guarantees for CRL to shed light on these findings.}}\label{figure 2 label} \vspace{-7pt}
\end{figure*}

\section{Empirical and Theoretical Insights for Continual Representation Learning}\label{sec:crl}
In this section, we discuss continual learning from the representation learning perspective. {We first present our experimental insights which show that (1) features learned from previous tasks help reduce the sample complexity of new tasks and (2) the order of task sequence (in terms of diversity and sample size) is critical for the success of CRL. In Section~\ref{sec:crl_theory}, we present our theoretical framework and a rigorous analysis in support of our experimental findings.}


\subsection{Empirical Investigation of CRL}\label{sec:crl_exp}
We further elucidate upon Figure \ref{figure 2 label} and discuss the role of sample size (for both past and new tasks) and task diversity.

\noindent\textbf{Investigating data efficiency (Fig \ref{fig:LS})} A good test to assess benefit of CRL is by constructing settings where new tasks have fewer samples. Consider SplitCIFAR100 for which $100$ classes are randomly partitioned into $20$ tasks. We partition the tasks into two sets: a continual learning set $\Dc_{cl}=\{\Tc_1,\dots,\Tc_{15}\}$ and a test set $\Dc_t=\{\Tc_{16},\dots,\Tc_{20}\}$. Test set is used to assess data efficiency and therefore, (intentionally) contains only 10\% of the original sample size (250 samples per task instead of 2500). We first train the network sequentially using  $\Dc_{cl}$  via ESPN/PackNet and create checkpoints of the supernetwork at different task IDs: At time $t$, we get a supernet trained with $\Tc_1,\dots,\Tc_t$, for $t\leq15$. {$t=0$} stands for the initial supernet without any training. Then we assess the representation quality of different supernet by individually training tasks in $\Dc_t$ on it. Fig.~\ref{fig:LS} displays the test accuracy of tasks in $\Dc_t$ where we used SplitCIFAR100 setting detailed in Sec.~\ref{sec:setting}. Figure~\ref{fig:LS} shows that ESPN-0.2,  ESPN-1, and PackNet methods all benefit from features trained by earlier tasks since the accuracy is above $75\%$  when we use supernets trained sequentially with multiple tasks. In contrast, individual learning trains separate models for each task where no knowledge is transferred; the accuracy is close to  $68\%$. Note that the performance of ESPN gradually increases with the growing number of continual tasks. This reveals its ability to successfully transfer knowledge from previous learned tasks and reduce sample complexity.

\noindent\textbf{Importance of task order and   diversity.} To study how task order and diversity benefits CL, similar to \cite{mallya2018packnet,hung2019compacting}, we use $6$ image classification tasks, where ImageNet-1k~(\cite{krizhevsky2012imagenet}) is the first task, followed by CUBS~(\cite{wah2011caltech}), Stanford Cars~(\cite{krause20133d}), Flowers~(\cite{nilsback2008automated}), WikiArt~(\cite{saleh2015large}) and Sketch~(\cite{eitz2012humans}). Intuitively, ImageNet should be trained first because of its higher diversity. 
Figure~\ref{fig:diversity} shows the accuracy improvement on the $5$ tasks that follow ImageNet pretraining compared to individual training. 
The results are displayed in Fig.~\ref{fig:diversity} where Green bars are CL with ImageNet as the first task, Orange is CL without ImageNet, and Blue is Individual training. In essence, this shows the importance of initial representation diversity in CL since results with the ImageNet pretraining (Green) are consistently and strictly better than no pretraining (Orange) and Individual (Blue).
We note that related experiments for zero-forgetting CL have been reported in \cite{hung2019compacting,mallya2018piggyback,mallya2018packnet,mallya2018piggyback,tu2020extending}. Unlike these works, experiments in Figure~\ref{fig:diversity} aim to isolate the representation learning benefit of the ImageNet dataset by training other 5 tasks continually with/without ImageNet. {We defer implementation details to the appendix.}

%


	

\noindent\textbf{Importance of sample size.}
Finally, we show that the sample size is also critical for CRL because one can build higher-quality (less noisy) features with more data. To this end, we devise another experiment based on SplitCIFAR100 dataset. Instead of using original tasks each with 5 classes and 2,500 samples, we train the first task (task ID 1 in Fig.~\ref{fig:Sample}) with 2,500 samples, then decrease the sample size for all the following tasks using the rule $\floor{2,500\times (1/20)^{\text{(ID-1)}/19}}$ until the last task (task ID 20 in Fig.~\ref{fig:Sample}) has only 125 training samples. Figure~\ref{fig:Sample} presents the results, where solid curves are obtained by training Task ID 1 to 20 with decreasing sample size, dashed curves are for training Task ID 20 to 1 with increasing sample size, and dotted curves are for individual training where task order does not matter. The accuracy curves are smoothed with a moving average and displayed in the decreasing order from Task ID 1 to 20. The results support our intuition that training large sample tasks first (decreasing order) performs better, as larger tasks build high quality representations that benefit generalization for future small tasks with less data. More strikingly, the dotted Individual line falls strictly between solid and dashed curves. This means that increasing order actually \emph{hurts accuracy} whereas decreasing order \emph{helps accuracy} compared to training from scratch (i.e.~no representation). Specifically, decreasing helps (solid$>$dotted) on the right side of the figure where tasks are small (thanks to good initial representations) whereas increasing hurts on the left side where tasks are large. The latter is likely due to the fact that, adding a large task requires a larger/better subnetwork to achieve high accuracy, however, since we train small tasks first, supernetwork runs out of sufficient free weights for a large subnetwork. 

\subsection{Theoretical Analysis and Performance Bounds for Continual Representation Learning}\label{sec:crl_theory}

In this section, we provide theoretical analysis to explain how CRL provably promotes sample efficiency and benefits from initial tasks with large diversity and sample size.

Denote $[n]=\{1,2,\dots,n\}$. We use $\ordet{\cdot}$ to denote equality up to a factor involving at most logarithmic terms. Following our experiments as well as \cite{maurer2016benefit}, a realistic model for deep representation learning is the compositional hypothesis $f=h\circ \phi$ where $h\in\Hb$ is the classifier head and $\phi\in\bPhi$ is the shared backend feature extractor. In practice, $\phi$ has many more parameters than $h$. To model continual learning, let us assume that we already trained a frozen feature extractor $\pf\in\bPf$\ylm{ on earlier tasks {from which $\phi$ can be learned faster}}. {Here $\Hb,\bPhi,\bPf$ are the hypothesis sets to learn from.} Suppose we are now given a set of $\Tn$ new tasks represented by independent datasets $\Sc_t=\{(\x_{ti},y_{ti})\}_{i=1}^N{\subset \Xc\times \Yc}$, each drawn i.i.d.~from different distributions $\Dc_t$ for $1\leq t\leq T$. \ylm{$\Xc,~\Yc$ are the sets of feasible input features and labels respectively.} Our goal is to build the hypotheses $(f_t)_{t=1}^{\Tn}:\Xc\rightarrow\R$ for these new tasks with small sample size $N$ while leveraging $\pf\in\bPf$. 

 \textbf{CRL setting.} In a realistic CL setting the new tasks are allowed to learn new features. We will capture this with an \emph{incremental} feature extractor {$\pn\in \bPn$, and represent the hypothesis of each task via the composition $f_t=h_t\circ \phi$ where $\phi=\pn+\pf$. For PackNet/ESPN, $\bPn$ corresponds to the free/trainable weights allocated to the new task, $\pf$ corresponds to the trained weights of the earlier tasks and $\phi$ corresponds to the eventual task subnetwork and its weights. We will evaluate the quality of $\phi$ (which lies in the Minkowski sum $\bPn+\bPf$) with respect to a global representation space $\bPhi$ which is chosen to be a superset: $\bPn+\bPf\subseteq \bPhi$. For instance, in PackNet, $\bPn+\bPf$ denotes the sparse sub-networks allocated to the \ylm{new and previous} tasks whereas $\bPhi$ corresponds to the full supernet.}



This motivates us to pose a CRL problem that builds a continual representation by searching for $\pn$ and combining with $\pf$. Let $\hb=(h_t)_{t=1}^\Tn\in \Hb^\Tn$ denote all $\Tn$ task-specific classifier heads, we solve
\begin{align}
\underset{\phi=\pn+\pf}{\underset{\hb\in\Hb^\Tn,\pn\in\bPn}{\arg\min}}&\Lch(\hb,\phi):=\frac{1}{\Tn}\sum_{t=1}^\Tn \Lch_{\Sc_t}(h_t\circ\phi)\nn\\
\text{WHERE}\quad &\Lch_{\Sc_t}(f):=\frac{1}{N}\sum_{i=1}^N \ell(y_{ti},f(\x_{ti})) \tag{CRL}\label{crl}.
\end{align}
\noindent \textbf{Intuition:} \eqref{crl} aims to learn the task-specific headers $\hb$ and the shared incremental representation $\pn$. Let $\cc{\cdot}$ be a complexity measure for a function class (e.g.~VC-dimension). Intuitions from the MTL literature would advocate that when the total sample size obeys $N\times \Tn\gtrsim \Tn\cc{\Hb}+\cc{\bPn}$, then \eqref{crl} would return generalizable solutions $\hhb,\pnh$. This is desirable as in practice $\pf$ is a much more complex hypothesis obtained by training on many earlier tasks. Thus, from continual learning perspective, theoretical goals are:
\begin{enumerate}
    \item The sample size should only depend on the complexity $\cc{\bPn}$ of the incremental representation rather than the combined complexity that can potentially be much larger ($\cc{\bPn}+\cc{\bPf}\gg \cc{\bPn}$).
    \item To explain Figure \ref{fig:LS}, we would like to quantify how frozen representation $\pf$ can provably help accuracy. Ideally, thanks to $\pf$, we can discover a near-optimal $\phi$ from the larger hypothesis set $\bPhi$.
    \item Finally, we emphasize that, we add the $T$ new tasks to the network in one round for the sake of cleaner exposition. \app{In Section~\ref{app C}, we provide synergistic theory and detailed investigation of the scenario where tasks are learned sequentially in a continual fashion and frozen features $\pf$ evolve as we add more tasks.} In a nutshell, this theory explains Figure \ref{fig:Sample} by quantifying the role of sample size in the quality of continual representations.
\end{enumerate}


Before stating our technical results, we need to introduce a few definitions. To quantify the complexities of the search spaces $\bPn,\Hb$, we introduce \emph{metric dimension} (\cite{mendelson2003few}), which is a generalization of the VC-dimension (\cite{vapnik2015uniform}). 

\begin{definition}[Metric dimension] \label{def:cov} Let $\Gb:\Zc\rightarrow\Zc'$ be a set of functions. {Let $\BC_{\Zc}>0$ be a scalar that is allowed to depend on $\Zc$.} Let $\Gb_{\eps}$ be a minimal-size $\eps$-cover of $\Gb$ such that for any $g\in\Gb$ there exists $g'\in \Gb_{\eps}$ that ensures $\sup_{\z\in\Zc}\tn{g(\z)-g'(\z)}\leq \eps$. The metric dimension $\cc{\Gb}$ is the smallest number that satisfies $\log|\Gb_{\eps}|\leq \cc{\Gb}\log(\BC_{\Zc}/\eps)$ for all $\eps>0$.
\end{definition}
{$\BC_{\Zc}$ typically depends only logarithmically on the Euclidean radius of the feature space under mild Lipschitzness conditions, thus, $\BC_{\Zc}$ dependence will be dropped for cleaner exposition.} In practice, for neural networks or other parametric hypothesis, metric dimension is bounded by the number of trainable weights up to logarithmic factors (\cite{barron2018approximation}). 

Metric dimension will help us characterize the sample complexity. However, we also would like to understand when $\bPf$ can help. To this end, we introduce definitions that capture the population loss (infinite data limit) of new tasks and the \emph{feature compatibility} between the new tasks and $\pf$ of old tasks. These definitions help decouple the finite sample size $N$ and the distribution of the new tasks.
\begin{definition}[Distributional quantities]\label{def pop}Define the population (infinite-sample) risk as $\Lc(\hb,\phi)=\E[\Lch(\hb,\phi)]$. Define the optimal risk over representation $\bPhi$ as $\Lco=\min_{\hb\in\Hb^\Tn,\phi\in\bPhi}\Lc(\hb,\phi)$. Note that the optimal risk can always choose the best representations within $\bPn$ and $\bPf$ since $\bPn+\bPf\subseteq \bPhi$. Finally, define the optimal population risk using frozen $\pf$ to be $\Lco_{\pfh}=\min_{\hb\in\Hb^\Tn,\pn\in\bPn}\Lc(\hb,\phi)$ s.t.~$\phi=\pn+\pf$.
\end{definition}
Following this, the \emph{representation mismatch} introduced below assesses the suboptimality of $\pf$ for the new task distributions compared to the optimal hypothesis within $\bPhi$.
\begin{definition}[New \& old tasks mismatch]\label{def MM} 
{The representation mismatch between the frozen features $\pf$ and the new tasks} is defined as \vspace{-8pt}
\begin{align*}
    \MM=\Lco_{\pf}-\Lco.
\end{align*}
\end{definition}
By construction, $\MM$ is guaranteed to be non-negative. Additionally, $\MM=0$ if we choose global space to be $\bPhi=\bPn+\bPf$ and $\pf$ to be the optimal hypothesis wihin $\bPf$. With these definitions, we have the following generalization bound regarding \eqref{crl} problem. 
\begin{theorem}\label{cl thm} Let $\hb,\hhb$ denote the set of classifiers $(h_t)_{t=1}^\Tn,(\hh_t)_{t=1}^\Tn$ respectively and $(\hhb,\hat{\phi}=\pnh+\pf)$ be the solution of \eqref{crl} given $\pf$. Suppose that the loss function $\ell(y,\hat{y})$ takes values on $[0,1]$ and is $\Gamma$-Lipschitz w.r.t.~$\hat{y}$. {Suppose that input set $\Xc$ is bounded and all $\pn\in\bPn,~h\in\Hb$, \ylm{$\bPn\in\{\bPn^i,1\leq i\leq k\}$,} and $\pf$ have Lipschitz constants upper bounded with respect to Euclidean distance}. With probability at least $1-2e^{-\tau}$, the task-averaged population risk of the solution $(\hhb,\hat{\phi})$ obeys 
\red{
\begin{align}
\Lc(\hhb,\hat{\phi})&\leq \Lco_{\pf}+ \sqrt{\frac{\Lco_{\pf}\cdot\ordett{\Hb,\bPn}}{\Tn N}}+\frac{\ordett{\Hb,\bPn}}{\Tn N},\nn\\
&\leq \Lco+{\MM}+{\sqrt{\frac{(\Lco+{\MM})\cdot\ordett{\Hb,\bPn}}{\Tn N}}}+\frac{\ordett{\Hb,\bPn}}{\Tn N},\nn
\end{align}
where $\ordett{\Hb,\bPn}:=\ordet{\Tn\cc{\Hb}+\cc{\bPn}+\tau}$}.
\end{theorem} 
\red{\textbf{Proof sketch:} Our proof uses a covering argument following Definition~\ref{def:cov} to approximate continuous search spaces $\Hb$ and $\bPn$ via discrete sets $\Hb_\eps$ and $\bPne$. With properly sized covers, for any $\hb\in\Hb^T, \pn\in\bPn$, we can find neighbors $\hb'\in\Hb_\eps^T,\pn'\in\bPne$ such that $|\Lc(\hb,\phi)-\Lc(\hb',\phi')|, |\Lch(\hb,\phi)-\Lch(\hb',\phi')|\leq\Gamma(L+1)\eps$, where $\phi=\pn+\pf$, $\phi'=\pn'+\pf$, and we assume $\pn,\pf,h$ are $L$-Lipchitz. We then apply Bernstein's inequality over all cover elements $\fb=(\hb',\phi')$, and derive that with probability at least $1-2e^{-\tau}$, we have that $|\Lc(\fb)-\Lch(\fb)|\leq\sqrt{\frac{D_\eps\cdot\Lc(\fb)+\tau}{TN}}+\frac{D_\eps+\tau}{TN}$. Here $D_\eps:=\Tn\cc{\Hb}\log(\bar C_\Hb/\eps)+\cc{\bPn}\log(\bar C_{\bPn}/\eps)$ captures dependence on dimension and specific hypothesis, namely, the bound is localized and smaller $\Lc(\fb)$ leads to a tighter bound. Finally, Theorem \ref{cl thm} is proved by combining both inequalities and setting $\eps=\frac{B(D_\eps+\tau)}{2\Gamma(L+1)TN}$. Further proof details are deferred to Appendix \ref{app B}.}

In words, this theorem shows that as soon as the total sample complexity obeys $\Tn N\gtrsim \Tn \cc{\Hb}+\cc{\bPn}$, we achieve small excess statistical risk and avoid the sample cost of learning $\pf$ from scratch. {Importantly, the sample cost of learning the incremental representation is shared between the tasks since per-task sample size $N$ only needs to grow with $\cc{\bPn}/T$.} 
%
Reusing $\bPf$ comes at the cost of a prediction bias $\MM$ arising from the feature mismatch. Also, with access to a larger sample size (e.g.~$\Tn N\gtrsim \Tn\cc{\Hb}+ \cc{\bPhi}$), new tasks can learn a near-optimal $\phi^\st\in\bPhi$ from scratch\footnote{In this statement, we ignore the continual nature of the problem and allow $\pf$ to be overridden for the new tasks if necessary.}. Thus, the benefit of \ref{crl} on data-efficiency is most visible when the new tasks have few samples, which is exactly the setting in Figure \ref{fig:LS}.   

\noindent$\bullet$ \textbf{$\bPn$ \& representation diversity.} Imagine the scenario where $\pf$ is already very rich and approximately coincides with the optimal hypothesis within the global space $\bPhi$. This is intuitively the ImageNet setting of Figure \ref{fig:diversity}\ylm{ where even fine-tuning $\pf$ will achieve respectable results}. Mathematically, this corresponds to the scenario where \red{$\bPn\approx \emptyset$ (compared to the $\bPf$ and sample size $NT$)} but the mismatch is $\MM\approx 0$. In this case, our theorem reduces to the standard few-shot learning risk where the only cost is learning $\Hb$ i.e.~\red{$\Lc(\hhb,\hat{\phi})\leq \Lco+ \sqrt{{\Lco\cdot\ordet{\cc{\Hb}}}/{N}}+{\ordet{\cc{\Hb}}}/{N}$}.


\noindent$\bullet$ \textbf{$\MM$ \& initial sample size.} Note that $\pf$ is built using previous tasks which have finite samples. \app{The sequential CL analysis we develop in {Section~\ref{app C}} decomposes mismatch as $\MM\lesssim\MM^\st+\red{\sqrt{\ordet{\cc{\bPf}}/N_{\text{prev}}}}$ where $\MM^\st$ is the mismatch if previous tasks had $N_{\text{prev}}=\infty$ samples and $\ordet{\cc{\bPf}/N_{\text{prev}}}$ is the excess mismatch due to finite samples shedding light on Figure \ref{fig:Sample}. }

\noindent\red{$\bullet$ \textbf{$\MM$ \& fast rates.} The square-root term in Theorem~\ref{cl thm} also shows that larger representation mismatch $\MM$ might result in a slower statistical learning rate. Consider the scenario where $\Lco\approx0$. If the new tasks have higher similarity with the previous tasks and can fully utilize the frozen features, that is $\MM\approx0$, then it achieves fast rate, i.e. $\Lc(\hhb,\hat{\phi})\lesssim{\ordet{T\cc{\Hb}+\cc{\bPf}}}/{NT}$. In contrast, when $\MM>0$ and $N$ is large, the square-root term is not negligible and $\Lc(\hhb,\hat{\phi})\lesssim\MM+\sqrt{{\MM\cdot\ordet{T\cc{\Hb}+\cc{\bPf}}}/{NT}}$. Therefore, the later-coming tasks (intuitively with higher mismatch) may incur slower statistical rates. 
}


Our analysis is related to the literature on representation learning theory (\cite{maurer2016benefit,kong2020meta,wu2020understanding,du2020few,gulluk2021sample,tripuraneni2020theory,arora2019theoretical}). Unlike these works, we consider the CL setting and show how the representation learned by earlier tasks provably helps learning the new tasks with fewer samples and how initial representation diversity and sample size benefit CRL. \red{We also show that representation mismatch is a significant term determining the statistical rate. In the following discussions, we expand our results in two ways: Section \ref{app:application} instantiates our result for neural networks (see Corollary~\ref{cl thm3}) to obtain tight sample complexity bounds (in the degrees of freedom) and Section~\ref{app C} provides bounds for learning tasks sequentially (see Theorem~\ref{seq thm}).}

\section{Application to Shallow Networks}\label{app:application}
As a concrete instantiation of Theorem \ref{cl thm} let us consider a realizable regression setting with a shallow network. More sophisticated examples are deferred to future work. Fix positive integers $d$ and $r_\frz\leq r$. Here, $d$ is the raw feature dimension and $r$ is the representation dimension which is often much smaller than $d$. The ingredients of our neural net example are as follows.
\begin{myenumerate}
\item Let $\psi:\R\rightarrow\R$ be a Lipschitz activation function with $\psi(0)=0$ such as Identity or (parametric) ReLU.
\blue{\item Let $\sigma:\R\rightarrow[-1,1]$ be a Lipschitz link function such as logistic function $1/(1+e^{-x})$.}
\item Let $Z$ be a zero-mean noise variable taking values on $[-1,1]$.
\item Fix vectors $(\vb^\st_t)_{t=1}^T\in\R^r$ with $\ell_2$ norms bounded by some $\bz>0$.
\red{\item Fix matrix $\W^\st\in\R^{r\times d}$ with Frobenius norm upper bounded by some $\bo>0$.}
\item Given input $\x\in\R^d$, task $t$ samples an independent $Z$ and assigns the label
\begin{align}
y=\sigma({\vb^\st_t}^\top \psi(\W^\st\x))+Z.\label{planted}
\end{align}
\item Fix $r_\frz\leq r$. Let $\W_\frz\in\R^{r\times d}$ be the matrix where the first $r_\frz$ rows \blue{are same as $\W^\st$}\ylm{This assumes no mismatch} whereas the last $r_\new:=r-r_\frz$ rows are equal to zero.
\end{myenumerate}
This setting assumes that first $r_\frz$ features are generated by $\W_\frz$ and \eqref{crl} should learn remaining features $\W^\st_\new:=\W^\st-\W_\frz$ and the classifier heads $(\vb^\st_t)_{t=1}^T$. We remark that above one can use arbitrary $[-C,C]$ limits rather than $[-1,1]$ or one can replace $[-1,1]$ limits with subgaussian tail conditions.

\blue{For some $\bo\geq \|\W^\st\|_F$}, we choose search space $\Wc$ to be the set of all matrices $\W_\new$ such that spectral norm obeys $\|\W_\new\|\leq \bo$ and the first $r_\frz$ rows of $\W_\new$ are zero. This way, we focus on learning the missing part of the representation $\W^\st$. Let us fix the loss function $\ell$ to be quadratic and denote $\Vb=(\vb_t)_{t=1}^T$. Then, \eqref{crl} takes the following parametric form
\begin{align}
&\hat{\Vb},\hat{\W}_\new=\underset{\W=\W_\frz+\W_\new}{\underset{\tn{\vb_t}\leq \bz,\W_\new\in \Wc}{\arg\min}}\Lch(\Vb,\W):=\frac 1 T\sum_{t=1}^\Tn \Lch_{\Sc_t}(\vb_t,\W) \nn\\
&\text{WHERE}\quad \Lch_{\Sc_t}(\vb_t,\W)=\frac{1}{N}\sum_{i=1}^N (y_{ti}-\sigma(\vb_t^\top \psi(\W\x_{ti})))^2. \nn
\end{align}
We have the following result regarding this optimization. It is essentially a corollary of Theorem \ref{cl thm}. The proof is deferred to the Appendix \ref{app nn proof}.
\begin{corollary}\label{cl thm3} Consider the problem above with $T$ tasks containing $N$ samples each with datasets $(\Sc_t)_{t=1}^T$ generated according to \eqref{planted}. Suppose input domain $\Xc\in\R^d$ has bounded $\ell_2$ norm. With probability at least $1-2e^{-\tau}$, the task-averaged population risk of the solution $(\hat{\Vb},\hat{\W}_\new)$ obeys 
\red{
\begin{align}
\Lc(\hat{\Vb},\hat{\W}_\new)\leq \left(\sqrt{\E[Z^2]}+ \sqrt{\frac{\ordet{\Tn r+r_\new d+\tau}}{\Tn N}}\right)^2.\nn
\end{align}}
\end{corollary}
\textbf{Interpretation:} Observe that the minimal risk is $\Lc^\st=\E[Z^2]$ which is the noise independent of features. The additional components are the excess risk due to finite samples. In light of Theorem \ref{cl thm}, we simply plug in $\cc{\Hb}=r$, $\cc{\bPn}=r_\new d$ and $\Lc^\st_{\pf}=\Lc^\st$. The first two arise from counting number of trainable parameters: each classifier has $r$ parameters and representation $\bPn$ has $r_\new d$ parameters. $\Lc^\st_{\pf}=\Lc^\st$ arises from the fact that we chose $\W_\frz$ to be subset of $\W^\st$ thus there is no mismatch. \blue{When $\W_\frz=0$} (i.e.~learning representation from scratch), this bound is comparable to prior works (\cite{tripuraneni2020provable,du2020few}), and in fact, it leads to (slightly) improved sample-complexity bounds. \blue{For instance, when $\psi$ is identity activation (i.e.~linear setting) \cite{tripuraneni2020provable} requires $TN\gtrsim r^2d$ samples to learn the task whereas our sample size grows only linear in $r$ and requires $TN\gtrsim rd$. Additionally, \cite{du2020few} requires per-task sample size to obey $N\gtrsim d$ samples whereas we only require $N\gtrsim r$.}

\section{Theoretical Analysis of Adding $T$ New Tasks Sequentially}\label{app C}

Theorem \ref{cl thm} adds $T$ tasks simultaneously on frozen feature extractor $\pf$. 
Below, we consider the setting where we add these tasks sequentially and a new task $t$ builds upon the cumulative representation learned from tasks $1$ to $t-1$. This setting better reflects what actually happens in continual learning and in our experiments but is more involved because representation quality of an earlier task will impact the accuracy of the future tasks. 

An intuitive way is to repeatedly apply Theorem~\ref{cl thm}. However, the population/empirical \red{excess risks} ($\Lc(\hhb,\hat{\phi})- \Lco_{\pf}/\Lc(\hhb,\hat{\phi})- \Lco$) shown in Theorem~\ref{cl thm} are coupled with $\Lco_{\pf}$, which are different for different tasks due to the evolving $\pf$ and various task distributions. \red{In order to provide a cleaner analysis} for the sequential setting, we will utilize the following bound that follows as a corollary of Theorem~\ref{cl thm} (see Cor.~\ref{cl thm variant}),\vspace{-2pt}
\[
\Lc(\hhb,\hat{\phi})\leq \Lco_{\pf}+{\sqrt{\frac{\ordet{\Tn\cc{\Hb}+\cc{\bPn}+\tau}}{\Tn N}}}.
\]
\red{While this statistical rate is slower, it has the advantage that the excess risk is decoupled from $\Lco_{\pf}$ thanks to its simpler form. To proceed, we first describe the sequential setting and assumptions on the representation mismatch.}

\noi\textbf{Sequential learning setting:} We learn a new task with index $t$ over the hypothesis set $\bPn^t$ for $t\in[T]$. Suppose we are at task $t$, that is, we assume that we have already built incremental (continual) feature-extractors $\pn^1,\dots,\pn^{t-1}$ for tasks $1$ through $t-1$ where each one obeys $\pn^\tau\in\bPn^\tau$. Thus, the (cumulative) frozen representation at time $t$ is given by
\[
\pf^t=\pf+\sum_{\tau=1}^{t-1} \pn^\tau.
\]
Here $\pf:=\pf^1\in\bPf$ is the representation built before any new task arrived. Using $\pf^t$, we solve the following (essentially identical) variation of \eqref{crl} where \textbf{we focus on task $t$ given the outcome of the continual learning procedure until task $t-1$.}\vspace{-4pt}
\begin{align}
h^t,&\pn^t=\underset{\phi=\pn+\pfh^t}{\underset{h\in\Hb,\pn\in\bPn^t}{\arg\min}} \Lch_{\Sc_t}(f):=\frac{1}{N}\sum_{i=1}^N \ell(y_{ti},f(\x_{ti}))\quad\text{where}\quad f=h\circ\phi. \tag{CRL-SEQ}\label{crlseq}
\end{align}
After obtaining $\pn^t$, the feature-extractor of task $t$ is given by $\phi^t=\pf^t+\pn^t$ and prediction function becomes $f^t=h^t\circ \phi^t$. Finally, $\phi^t$ of task $t$ becomes the next frozen feature-extractor i.e.~$\pf^{t+1}=\phi^t$.

In this sequential setting, intuitively $(\bPn^t)_{t=1}^T$ are less complex hypothesis spaces compared to $\bPn$ of Theorem \ref{cl thm}. This is because we learn $\bPn^t$ using a single task. In that sense, the proper scaling of hypothesis set complexity is $\cc{\bPn^t}\propto \cc{\bPn}/T$ for $t\in[T]$. Specifically, we assume that for some global value $\Ccn>0$
\begin{align}
    \cc{\bPn^t}\leq \Ccn\quad \text{for all}\quad 1\leq t\leq T.\label{ccn decay}
\end{align}
Secondly, compared to Theorem \ref{cl thm}, we need to introduce a more intricate compatibility condition to assess the benefit of the representations learned from finite data $\pn^1,\dots,\pn^{t-1}$ for the new task $t$. This will be accomplished by first introducing population level compatibility and then introducing an assumption that controls the impact of finite sample learning on the new task. The definition and assumption arise naturally to control the learnability of a new task given features of earlier tasks. Related assumptions (e.g.~\emph{task diversity} condition) have been used by other works for transfer/meta learning purposes (\cite{tripuraneni2020theory,oymak2021generalization,du2020few,xu2021representation}).

Let $(h^{\st,1},\pn^{\st,1}),\dots,(h^{\st,t},\pn^{\st,t}),\dots$ be the (classifier, representation) sequence obtained by solving \eqref{crlseq} using infinite samples $N=\infty$ (that is, solving the population-level optimization rather than finite-sample ERM). The following definition introduces the representation mismatch at task $t$ to capture the suitability of the population-level representations $(\pn^{\st,\tau})_{\tau=1}^{t-1}$ for a new task $t$. Set $\pf^{\st,t}=\pf+\sum_{\tau=1}^{t-1}\pn^{\st,\tau}$ and define the set of all feasible representations for task $t$ as \blue{$\bPhi^t=\sum_{\tau=1}^t\bPn^\tau+\bPf${$\subseteq\bPhi$}.}
\begin{definition}[Population quantities and representation mismatch]\label{def pop seq}For task $t$, define the population (infinite-sample) risk as $\Lc_t(h,\phi)=\E[\Lch_{\Sc_t}(h,\phi)]$. Define the optimal risk of task $t$ over all feasible representations in {$\bPhi$} as $\Lci{t}=$ $\min_{h\in\Hb,\phi\in{\bPhi}}\Lc_t(h,\phi)$. Note that the optimal risk gets to choose the best representations within $(\bPn^\tau)_{\tau=1}^t$ and $\bPf$ {since $\sum_{\tau=1}^t\bPn^\tau+\bPf\subseteq \bPhi$.} Finally, define the optimal risk with fixed frozen model $\pf^{\st,t}=\pf+\sum_{\tau=1}^{t-1}\pn^{\st,\tau}$ to be $\Lci{t}_\seq=\min_{h\in\Hb,\pn^t\in\bPn^t}\Lc_t(h,\phi)$ s.t.~$\phi=\pn^t+\pf^{\st,t}$. The sequential representation mismatch at task $t$ is defined as
\begin{align}
\MS{t}=\Lci{t}_\seq-\Lci{t}.\label{mst}
\end{align}
\end{definition}
This definition quantifies the cost of continual learning with respect to choosing the best (oracle) representation. It also aims to capture the properties of the task distributions thus it uses infinite samples for tasks $1\leq \tau\leq t$. In practice, a new task $t$ is learned on top of finite sample tasks. We need to make a plausible assumption to formalize
\begin{quote}
\emph{With enough samples, representations learned from finite sample tasks are almost as useful as representations learned from infinite sample tasks.}
\end{quote}
We accomplish this by introducing empirical/population compatibility below. The basic idea is that, quality of the representation should decay gracefully as we move from infinite to finite samples.
\begin{assumption} [Empirical/population compatibility] \label{fin comp} For task $t$, define the population risk $\Lc_t(h,\phi)=\E[\Lch_{\Sc_t}(h,\phi)]$. Recall the definitions of $\pf,(h^{\st,t},\pn^{\st,t})_{t\geq 1}$ from Def.~\ref{def pop seq}. Given a sequence of incremental feature-extractors $\bhi:=(\pn^\tau)_{\tau\in[t]}$, recall from \eqref{crlseq} that task $t$ uses the extractor $\phi^{t}=\pf^{t+1}=\pf+\sum_{\tau=1}^{t}\pn^{\tau}$. {To quantify representation quality}, we introduce the risks $\Lcb^t_\seq:=\Lcb^{t,\bhi}_\seq,\Lc^t_\seq:=\Lc^{t,\bhi}_\seq$ induced by $\bhi$ (similar to Def.~\ref{def pop seq})
\begin{align}
&\Lcb^t_\seq=\min_{h\in\Hb}\Lc_t(h,\phi^t)\nn\\
&\Lc^t_\seq=\min_{h\in\Hb,\pn\in\bPn^t}\Lc_t(h,\phi)~\text{s.t.}~\phi=\pf^t+\pn.\label{ltseq}
\end{align}
Here $\Lcb^t_\seq$ uses the given $\pn^t$ (within $\bhi$) whereas $\Lc^t_\seq$ chooses the optimal $\pn^t\in \bPn^t$, and they both assume frozen representation $\pf^t$\footnote{$\Lc^t_\seq$ definition is needed to quantify the representation quality of a new task where incremental update $\pn^t$ has not been built yet. In contrast, $\Lcb^t_\seq$ quantifies the representation quality for which (full) feature extractors are known.}. Thus $\Lcb^t_\seq\geq \Lc^t_\seq$. Based on these, define the mismatch between empirical and population-level optimizations for task $t$ as
\begin{align}\nn
\ME{t}=\Lc^t_\seq-\Lci{t}_\seq\quad\text{and}\quad\MA{t}=\Lcb^t_\seq-\Lci{t}_\seq.
\end{align}
Again by construction {$\MA{t}\geq \ME{t}$}. We say \textbf{empirical and population representations} are compatible if there exists a constant $\cz>0$ such that, for all choices of $(\pn^\tau)_{\tau=1}^t\in \bPn^1\times \dots\bPn^t$, we have that
\[
\blue{\underbrace{\ME{t}}_{\text{subopt on new task}}\leq\underbrace{\cz}_{\text{additive mismatch}}+\underbrace{\frac{1}{t-1} \sum_{\tau=1}^{t-1} \MA{\tau}}_{\text{avg subopt on first $t-1$ tasks}}}
\]
\end{assumption}
\textbf{Interpretation:} Here, \blue{$\frac{1}{t-1} \sum_{\tau=1}^{t-1} \MA{\tau}$} quantifies the suboptimality of the representations used by first $t-1$ tasks. Recall that task $\tau$ uses representation $\pf^\tau$ for $\tau\leq t-1$. Verbally, this assumption guarantees that, task $t$ can find an (incremental) representation $\pn^t\in\bPn^t$ and classifier $h\in\Hb$ such that its suboptimality to population-optimal risk $\Lci{t}_\seq$ is upper bounded in terms of the average of the suboptimalities over the first $t-1$ tasks. Note that, this can also be viewed as a \textbf{sequential task diversity} condition because we are assuming that good quality representations on the first $t-1$ tasks (w.r.t.~population minima) ensure a small excess risk (w.r.t.~population minima) on the new task.

Following Definition~\ref{def pop seq} and Assumption \ref{fin comp}, the lemma below probabilistically quantifies the generalization risk when we add one task. Using this lemma, we will state our main result which quantifies the generalization risk when adding $T$ tasks.
\begin{lemma}\label{lem seq} Suppose we are given the output pairs $(h^\tau,\pn^\tau)_{\tau=1}^{t-1}$ of the first $t-1$ applications of sequential CRL problem \eqref{crlseq}. Now, we solve for the $t$'th solution denoted by the pair $(h^t,\pn^t)$. Under same conditions as in Theorem \ref{cl thm} (Lipschitz hypothesis $\bPn^t,\Hb$, Lipschitz loss $\ell:\R\times \R\rightarrow [0,1]$ and bounded input features $\Xc$), for some absolute constant $C>0$, with probability $1-2e^{-\tau}$, the solution $(h^t,\pn^t)$ of \eqref{crlseq} satisfies the following two properties
\begin{align}\label{eq mm}
&\Lc(h^t,\pn^t+\pf^t)-\Lci{t}\leq \MS{t}+\ME{t}+\sqrt{\frac{\ordet{\cc{\Hb}+\cc{\bPn^t}+\tau}}{N}}\\
&\MA{t}\leq \sqrt{\frac{\ordet{\cc{\Hb}+\cc{\bPn^t}+\tau}}{N}}+\ME{t}.\nn
\end{align}
Here, $\MS{t},\ME{t},\MA{t}$ are mismatch definitions introduced in Definition~\ref{def pop seq} and Assumption \ref{fin comp} based on given incremental feature-extractor sequence of previous tasks $(\pn^\tau)_{\tau=1}^{t-1}$.
\end{lemma}

The following theorem is our main guarantee on sequential CRL. It is obtained by stitching $T$ applications of Lemma \ref{lem seq} which adds a single new task.
\begin{theorem}\label{seq thm} Suppose we solve the sequential continual learning problem \eqref{crlseq} for each task $1\leq t\leq T$ to obtain hypothesis $(h^t,\pn^t)_{t=1}^T$. The $t$'th model uses the prediction $h^t\circ \phi^{t}$ where $\phi^t=\pf+\sum_{\tau=1}^{t-1}\pn^\tau$. Consider the same core setting as in Theorem \ref{cl thm}: Namely, we assume Lipschitz hypothesis sets $\bPn^t,\Hb$, Lipschitz loss function $\ell:\R\times \R\rightarrow[0,1]$ and bounded input feature set $\Xc$ all with respect to Euclidean distance. Recall that $\Lci{t}$ is the optimal risk for task $t$. Suppose the complexity of each $\bPn^t$ is upper bounded by $\Ccn$ as in \eqref{ccn decay} for $t\in [T]$. For some absolute constant $C>0$, with probability $1-2Te^{-\tau}$, the solutions $(h^t,\pn^t)_{t=1}^T$ of \eqref{crlseq} satisfy the following cumulative generalization bound (summed over all $T$ tasks)
\begin{align}
\underbrace{\sum_{t=1}^T\left(\Lc_t(h^t,\phi^t)-\Lci{t}\right)}_{\text{excess test risk w.r.t.~oracle}}&\leq \underbrace{\sum_{t=1}^T\MS{t}}_{\text{sequential representation mismatch}}+\underbrace{T^2\left(\cz+\sqrt{\frac{\ordet{\cc{\Hb}+\Ccn+\tau}}{N}}\right)}_{\text{cost of finite sample learning}}.\label{gen bound 3}
\end{align}
In light of Def.~\ref{def pop seq}, we can write the suboptimality with respect to solving sequential problems with $N=\infty$ as
\begin{align}
\underbrace{\sum_{t=1}^T\left(\Lc_t(h^t,\phi^t)-\Lci{t}_\seq\right)}_{\text{excess test risk w.r.t.~sequential learning}}&\leq \underbrace{T^2\left(\cz+\sqrt{\frac{\ordet{\cc{\Hb}+\Ccn+\tau}}{N}}\right)}_{\text{cost of finite sample learning}}.\label{gen bound 4}
\end{align}
\end{theorem}
\noi \textbf{Discussion.} Here are a few remarks in place. First, we state the sum of test errors rather than the average. Secondly, observe that \eqref{gen bound 3} compares the test errors to the optimal possible errors $\Lci{t}$. On the right hand side there are two terms: ``representation mismatch'' and ``cost of finite sample learning''. 

``Representation mismatch'' quantifies the population-level error that arises even if each task had access to infinite samples. This is because, even if each task solved \eqref{crlseq} perfectly, the resulting sequence of representations does not have to be optimal for the next task and $\MS{t}$ precisely captures this suboptimality. Recall that this population-level gap arises from Definition \ref{def pop seq}. {This also emphasizes that we should train diverse tasks firstly, since diverse features learned from previous tasks help reduce $\MS{t}$ due to its highly relevant representation. }

The ``cost of finite sample learning'' term originally captures the finite sample effects, and it is proportional to the statistical error rate of solving \eqref{crlseq} for the first task-only i.e.~$\sqrt{\frac{\Cc(\Hb)+\Ccn}{N}}$. Here, recall from \eqref{ccn decay} that $\Ccn$ is an upper bound to the complexities of $\bPn^1,\dots,\bPn^T$. Perhaps unexpected dependence is the quadratic growth in $T$. This is in contrast to linear growth one would get from adding tasks simultaneously as in Theorem \ref{cl thm}. This quadratic growth arises from the accumulation of the finite-sample representation suboptimalities as we add more tasks. Specifically, Assumption \ref{fin comp} helps guarantee that feature-extractors of tasks $1,\dots,t-1$ are useful for task $t$; however, as more tasks are added they incur more divergence from $(\phi^{\st,t})_{t=1}^T$. Each new task has a finite sample size and contributes to this divergence and our analysis leads to $\order{T^2}$ upper bound on the error. $\cz$ is an additional mismatch term that makes Assumption \ref{fin comp} significantly more flexible (albeit ideally, it is close to zero). Finally, it would be interesting to explore the tightness of these bounds for concrete analytical settings (e.g.~\ref{crlseq} with linear models or neural nets), run more experiments and further study the role of finite sample effects \& representation divergences. 



\section{Inference-efficient Continual Representations via ESPN} \label{sec: espn_detail}
In this section, we introduce more implementation and evaluation details of our ESPN algorithm. We use the phrases \emph{mask} and \emph{sub-network} interchangeably because we obtain the task sub-network by masking weights of the supernet. This sub-network is the nonzero support of the task, that is, the task-specific model is obtained by setting other weights to zero. We assume a sequence of tasks $\{\Tc_t$, $1\leq t\leq T\}$ is received during training time, where $t$ is task identifier, $T$ is the number of tasks, and $\Tc_t$ has training dataset $\Sc_t=\{(\x_{ti},y_{ti})\}_{i=1}^{N_t}$ (In experiments, $(N_t)_{t=1}^T$ do not need to be the same.). 
Given task sequence and a single model, our goal is to find optimal sparse sub-networks that satisfy both FLOPs and sparsity restrictions without performance reduction and knowledge forgetting. FLOPs constraint is important for efficient inference whereas sparsity constraint is important for adding all tasks into the network even for large number of tasks $T$. {To this end, we use joint channel and weight pruning strategy.} Let $\ell$ be a loss function, $f$ be a hypothesis (prediction function) and $\bt\in\R^p$ denote the weights of $f$. We focus on task $\Tc_t$, assuming $\Tc_1,...,\Tc_{t-1}$ are already trained. Let $\m_t\subset[p]$ be the nonzero support of task $t$ and $\ma_t=\cup_{\tau=1}^{t-1} \m_\tau$ be the combined support until task $t-1$ (mask of frozen weights for training $\Tc_t$). Initially $\ma_1=\emptyset$. Let $\bt_t\in\R^p$ be the model weights at time $t$. Note that all the trained weights of the previous tasks ($\Tc_1\dots \Tc_{t-1}$) lie on the sub-network $\ma_t$. We use the notation $\bt\odot\m$ to set the weights of $\bt$ outside of the mask $\m$ to zero. 
Define the loss $\Lc_{\Sc_t}(\bt)=\frac{1}{N_t}\sum_{i=1}^{N_t}\ell(y_{ti},f(\x_{ti};\bt))$. The procedure for learning task $\Tc_t$ is formulated as the following optimization task that updates {supernet} weight/mask pair and {returns $(\bt_t,\m_t)$} given $(\bt_{t-1},\ma_{t})$:
\begin{align}
    \bt_t,\m_t=\arg&\min_{\bt,\m}~~~\Lc_{\Sc_t}(\bt\odot\m)\tag{ESPN-OPT}\label{cl_prob}\\
    \text{s.t.} 
    &~~~\FLOP(\m)\leq \OFLOP_t,\nn\\
    &~~~\mn=\m\setminus\ma_{t},\nn\\
    &~~~\NNZ(\mn)\leq \ONNZ_t,\nn\\
    &~~~{\bt\odot\ma_{t}}=\bt_{t-1}\odot\ma_{t}.\nn
\end{align}
Here $\m_t$ is the channel-constrained mask that corresponds to the sub-network of $\Tc_t$ and $\bt_t\odot\m_t$ are the weights we use for task $\Tc_t$ and the prediction function is $f(\cdot;\bt_t\odot\m_t)$. {The updated mask until task $t$ is obtained by $\ma_{t+1}=\m_t\cup \ma_{t}$, which is also seen as mask of frozen weights for training task $t+1$.} $\FLOP(\m)$ and $\NNZ(\m)$ return the FLOPs and nonzeros of a given mask $\m$. $\OFLOP_t$ and $\ONNZ_t$ are the FLOPs and nonzero constraints of task $t$. Observe that we only enforce NNZ constraint on the new weights $\mn$ whereas FLOPs constraint applies on the whole sub-network. {The last equation in \eqref{cl_prob} highlights that the weights of earlier tasks on $\ma_{t}$ are kept frozen}. 
In practice, FLOPs \& NNZ constraints in \eqref{cl_prob} lead to a combinatorial problem. We propose Algorithm \ref{algo 1} to (approximately) solve this problem efficiently which learns the new sub-network in three phases: pre-training (over all free weights), gradual pruning (to satisfy constraints and obtain $\mnt$), and fine-tuning (to refine the weights on $\mnt$). While not shown in Algorithm \ref{algo 1}, we also introduce the following innovations.

      \begin{algorithm}[t]\caption{Efficient Sparse PackNet (ESPN-{$\gamma$})}
    \begin{algorithmic}[1]
    \Require Task sequence $\{\Tc_t\}$; model weights $\bt$; step size $\eta$; pre-training, pruning, fine-tuning durations; \clr{FLOPs constraint $\gamma$; weight allocation parameter $\alpha$}.
    \State Set of task masks $\Mc\gets\emptyset$
    \State Set of new weights {$\mn\gets[p]$} \Comment{Initially all weights are new/free}
    \For {$\Tc_t \in \{\Tc_t\}$}
    \While {pre\_training}
    \State $\bt\leftarrow\bt-\eta\cdot\nabla\Lc_{\Sc_t}(\bt)\odot \mn$
    \Comment{Pre-train over all free weights}
    \EndWhile
    \State Initialize sub-network mask $\m_t\gets[p]$
    \While {\clr{FLOPs($\gamma$) \& weight\_sparsity($\alpha$) constraints not satisfied}} 
    \State Gradual channel \& weight pruning: update $\m_t$\label{algo:pruning}
    \State $\bt\leftarrow\bt-\eta\cdot\nabla\Lc_{\Sc_t}(\bt)\odot (\m_t\cap\mn)$
    \EndWhile
    \While {fine\_tuning}
    \State $\bt\leftarrow\bt-\eta\cdot\nabla\Lc_{\Sc_t}(\bt)\odot(\m_t\cap\mn)$
    \Comment{Fine-tune over selected free weights}
    \EndWhile
    \State $\Mc\leftarrow\Mc\cup\{t:\m_t\}$\Comment{Save the subnet mask of $\Tc_t$}
    \State $\mn\leftarrow\mn \setminus \m_t$ \Comment{Update the free weights set $\mn$\ylm{Remove the weights in $\m_t$}}
    \EndFor
    \end{algorithmic}\label{algo 1}
\end{algorithm}
%

\noindent$\bullet$ \textbf{Trainable task-specific BatchNorm.} 
We train separate BatchNorm layers for each task. This has multiple synergistic benefits. First, our algorithm trains faster and generalizes better than PackNet which does not train BatchNorm weights (see Table \ref{CIFARtable}). Specifically, training BatchNorm weights allows ESPN to re-purpose the (frozen) weights of the earlier tasks with negligible memory cost. BatchNorm weights also guide our channel pruning scheme described next.

\noindent$\bullet$ \textbf{FLOP-aware pruning.} Many of the prior works on channel pruning (\cite{liu2017learning,zhuang2020neuron}) focus on reducing the number of channels rather than the computation/FLOPs cost of the channels which varies across layers. {As {shown in Fig.~\ref{fig:CIFAR10_ratio}} in appendix, in order to satisfy FLOPs constraint $\gamma$, prior methods require numerous channels to be pruned}. This results in unsatisfactory performance especially under aggressive FLOPs constraint {(shown in Fig.~\ref{fig:CIFAR10_acc})}. However in CL setup, in order to train a single model with many tasks, a significantly larger supernet with high capacity is needed compared to a single task requirement, and we aim to find a sub-network with very few FLOPs without compromising performance. To fit our specific needs for channel pruning, in this paper we present an innovative channel pruning algorithm called {FLOP-aware channel pruning} that preserves the performance up to 80\% FLOPs reduction in our SplitCIFAR100 experiments (Table \ref{CIFARtable}). Following \cite{liu2017learning}, we consider BatchNorm weights as trainable saliency scores for convolutional channels and prune all channels with scores lower than a certain threshold by setting them to zero. Let $\bGam$ be the BatchNorm vectors. Given dataset $\{(\x_i,y_i)\}_{i=1}^N$, in practice, rather than solving the constrained problem \eqref{cl_prob}, Algorithm~\ref{algo 1} minimizes a regularized objective $\Lc_{\Sc_t}(\bt\odot\m)+\Rc(\bGam,\m)$. Here $\Rc$ is a regularization term to promote channel sparsity in \eqref{app eq:reg}. 


For an $L$ layer network, use $\bGam_l$ to denote the $l$'th {BatchNorm weights} for $l\in[L]$. Intuitively, we wish to use $\ell_1$-regularization on $\bGam$. However, since layers of the network show variation, layerwise regularization parameters $(\lambda_l)_{l=1}^L$ are needed. Instead of designing $\lambda_l$ by trial-and-error which is time-consuming and expert-dependent, we introduce a method that chooses $\lambda_l$ automatically, fine-tunes $\lambda_l$ during gradual pruning and adapts to the global FLOPs constraint.  
 Specifically, $\lambda_l$ is chosen based on the {FLOPs load of a channel} determined by its input feature dimensions and operations. Here an implicit goal is pruning as few channels as possible while achieving maximum FLOPs reduction. To achieve these goals, we use the following FLOP-weighted sparsity regularization 
 \begin{align}
	\label{app eq:reg}\Rc(\bGam,\m)&=\sum_{l=1}^L\lambda_l\|\bGam_l\|_1,~~\text{where}~~\lambda_l=g(\FLOP_l(\m)).
\end{align}
Here $\FLOP_l(\m)$ calculates the FLOPs load for a channel in the $l$'th layer of the subnet, and $g(\cdot)$ is a monotonically increasing function. Notably, channels in the same layer load the same number of FLOPs. We use $\ell_1$-penalty to enforce unimportant elements to zeros and prune the channels with the smallest weights over all layers. Since $g(\cdot)$ is increasing, channels costing more FLOPs are assigned with larger $\lambda_l$ and are pushed towards zero, thus they are easier to be pruned. Additionally, since $\FLOP_l(\cdot)$ is based on subnet $\m$, $\lambda_l$ is automatically tuned while we use gradual pruning (Line~\ref{algo:pruning} in Algorithm \ref{algo 1}). In our experiments, we use $g(x)=C\sqrt{x}$ for a proper scaling choice $C>0$. {Here $C$ can be seen as a normalized term and in detail we have
\begin{align*}
    \lambda_l=\frac{\sqrt{\FLOP_l(\m)}}{\sum_{i=1}^L\sqrt{\FLOP_i(\m)}}.
\end{align*}}





\noindent$\bullet$ \textbf{Weight allocation.} Since we do not modify the supernet architecture, without care, supernet might run out of free weights if there is a huge number of tasks. While original PackNet paper (\cite{mallya2018packnet}) also uses weight pruning, since they consider relatively fewer tasks, they don't develop an algorithmic strategy for allocating the free weights to new tasks. In our experiments, we introduce a simple weight allocation scheme to assign $\ONNZ_t$ {depending on the number of remaining free weights}. Let $p$ be the total number of weights, $p_t$ be the total number of weights used by tasks $1$ to $t$ and $p_0=0$. We set 
\begin{align}\label{wa-eq}
\ONNZ_t=\lceil (p-p_{t-1})\cdot\alpha\rceil\quad\text{for some}\quad 0<\alpha<1.
\end{align} Here $\alpha$ is {the \emph{weight-allocation} level} and a new task gets to use $\alpha$ fraction of all unused weights in the supernet. We emphasize that weight allocation controls the number of new nonzeros allocated to a task. A new task is allowed to use all of the (frozen) nonzeros that are allocated to the previous tasks (as long as FLOPs constraint is not violated).

\subsection{Experimental Setup}\label{sec:setting}
We evaluate the performance of our proposed ESPN in terms of accuracy and efficiency metrics on three datasets: SplitCIFAR100, RotatedMNIST, and PermutedMNIST. We compare ESPN to numerous baselines, which include training each task individually, multitask learning (MTL), PackNet~(\cite{mallya2018packnet}), CPG~(\cite{hung2019compacting}), RMN~(\cite{kaushik2021understanding}), and SupSup~(\cite{wortsman2020supermasks}). The last four methods are zero-forgetting CL methods. \red{To the best of our knowledge, none of the existing methods consider the inference-efficiency of individual tasks within the CL setting. In our experiments, in addition to embedding each task to a sparse subnet, we also enforce each subnet to be computation-efficient via our channel pruning technique while minimally harming accuracy. }

\noindent{\textbf{Datasets. }} SplitCIFAR100, RotatedMNIST and PermutedMNIST are popular datasets for continual learning that we also use in our experiments. We follow the same setting as in \cite{wortsman2020supermasks}. For SplitCIFAR100 dataset, we randomly split CIFAR100 (\cite{krizhevsky2009learning}) into 20 tasks where each task contains $5$ classes, $2500$ training samples, and $500$ test samples. RotatedMNIST is generated by rotating all images in MNIST by the same degree. In our experiments, we generate $36$ tasks with $10,20,\ldots, 360$ degree rotations {and train in a random order}. PermutedMNIST dataset is created by applying a fixed pixel permutation to all images, and we created $36$ tasks with independent random permutations. 

\begin{table*}[t]
    \caption{Continual learning results on SplitCIFAR100 tasks with ResNet18 model using different algorithms. 20 tasks are trained in order. Results for each task is the average accuracy over 5 independent runs. Last column shows the average accuracy for all the tasks with standard error of 5 runs. Individual-0.2/-1 and ESPN-0.2/-0.5/-1 denote the methods with different level of FLOPs constraints.}
    \centering\hspace{-7pt}
    \scriptsize
    \setlength{\tabcolsep}{0.36mm}{
    \begin{tabular}{lcccccccccccccccccccc|c}
    \midrule
    \hfill \textbf{Task ID} $\rightarrow$ & 1 & 2 & 3 & 4 & 5 & 6 & 7 & 8 & 9 & 10 & 11 & 12 & 13 & 14 & 15 & 16 & 17 & 18 & 19 & 20 & avg$\pm$ error\\
    \midrule
    CPG &80.2 &87.9 &92.0 &77.7 &87.7 &81.1 &88.2 &93.8 &87.8 &80.7 &83.3 &86.0 &90.9 &89.4 &81.4 &92.6 &87.2 &90.6 &88.8 &85.5 &86.6$\pm$0.27 \\
    
    RMN &78.5 &83.7 &86.6 &68.4 &85.3 &73.4 &77.8 &88.1 &79.7 &68.3 &71.2 &76.7 &85.2 &79.7 &66.8 &81.8 &70.7 &79.3 &71.9 &72.7 &77.3$\pm$0.07 \\

    PackNet &84.5 &89.6 &92.6 &78.7 &89.8 &83.7 &89.2 &94.4 &89.7 &83.5 &84.6 &87.6 &93.0 &91.9 &84.2 &93.2 &86.6 &91.6 &89.7 &86.8 &88.3$\pm$0.10\\
    
    SupSup &86.5 &91.8 &92.9 &81.0 &89.1 &83.4 &88.5 &95.1 &87.8 &83.0 &81.9 &89.0 &94.3 &93.4 &86.6 &94.0 &91.0 &92.6  &91.2 &88.3 &89.1$\pm$0.06\\
    
    \midrule
    MTL &89.8 &93.0 &95.2 &82.5 &91.0 &85.7 &92.0 &95.8 &91.3 &86.1 &86.3 &90.0 &94.1 &94.7 &86.7 &95.2 &93.6 &94.8 &92.9 &90.0 &91.0$\pm$0.21\\
    
     Individual-0.2 &86.4 &88.8 &93.0 &80.4 &87.5 &85.1 &85.6 &94.7 &85.7 &83.4 &82.3 &88.2 &92.5 &93.2 &85.0 &94.5 &90.5 &92.7 &91.7 &89.5 &88.5$\pm$0.07\\
   
    Individual-1 &86.2 &89.2 &93.2 &79.8 &88.6 &84.1 &85.4 &94.6 &87.0 &82.2 &82.8 &88.3 &93.4 &92.4 &85.6 &94.5 &90.8 &93.0 &92.0 &89.9 &88.6$\pm$0.10\\
    
    \midrule
    
    ESPN-0.2 &86.4 &90.8 &93.4 &82.4 &89.5 &83.1 &89.8 &95.2 &89.8 &83.5 &85.2 &88.6 &93.7 &92.8 &85.1 &94.8 &88.8 &91.8 &91.2 &86.8 &\textbf{89.1}$\pm$0.12\\
    
    
    ESPN-0.5 &86.8 &91.3 &94.2 &81.6 &89.2 &84.0 &90.0 &95.7 &90.3 &84.6 &85.3 &89.4 &94.1 &93.0 &85.5 &94.7 &90.2 &93.1 &92.2 &88.8 &\textbf{89.7}$\pm$0.11\\
    
    ESPN-1 &85.0 &91.0 &93.5 &82.3 &89.1 &84.4 &90.2 &95.8 &91.1 &85.2 &86.4 &89.6 &94.3 &93.8 &85.6 &95.4 &90.2 &93.0 &92.9 &88.0 &\textbf{89.8}$\pm$0.14\\
    \midrule
    \end{tabular}}\vspace{-7pt}\hspace{-7pt}\label{CIFARtable}
\end{table*}

\noindent{\textbf{Models and implementation. }} For SplitCIFAR100, {following \cite{lopez2017gradient} we use a variation of ResNet18 model with fewer channels.} For each task, we use a batch size of 128 and Adam optimizer (\cite{kingma2014adam}) with hyperparameters $(\beta_1,\beta_2)=(0,0.999)$. As shown in Algorithm~\ref{algo 1}, for each task we apply pretraining, pruning, and finetuning strategies. First, we pretrain the model for $60$ epochs with learning rate $0.01$. Then, we gradually prune the channels and weights within $90$ epochs using the same learning rate. For the finetuning stage, we apply cosine decay (\cite{loshchilov2016sgdr}) over learning rate, starting from $0.01$, and train for $100$ epochs. Therefore, each task trains for $250$ epochs in total. 

For RotatedMNIST and PermutedMNIST experiments, we use the same setting \blue{of FC1024 model} in \cite{wortsman2020supermasks}. This is a fully connected network with two hidden layers of size $1024$. We train for $10$ epochs with RMSprop optimizer (\cite{graves2013generating}), batch size of $256$, and learning rate $0.001$.  
The number of pretraining, pruning, and finetuning epochs are $3$, $4$, and $3$, respectively.


As for comparison baselines, MTL and individual training baselines follow the same configurations (e.g.~architecture, hyperparameters) as ESPN. In SplitCIFAR100 experiments, labels are not shared among different tasks. Thus, each task is assigned a separate classifier head while sharing the same backend supernet as a feature extractor. {Unlike SplitCIFAR100, all the tasks in RotatedMNIST and PermutedMNIST experiments share the same head via weight pruning.} MTL simultaneously trains all 20 tasks while sharing the backend supernet. In individual training, each task trains their own backend. Finally, CPG, RMN, and SupSup are all trained with their own publicly available codes but over the same ResNet18 model. To ensure fair comparison, we do not allow dynamic model expansion/enlargement in CPG. Similarly, for PackNet, instead of using its original setting with fixed pruning ratio for each task, we run it with our \emph{weight allocation} strategy to make it easier to compare. These baselines are all evaluated in Table~\ref{CIFARtable} which is discussed below.

\subsection{Investigation of Inference Efficiency}



\noindent{\textbf{SplitCIFAR100.}} Table~\ref{CIFARtable} presents results for our experiments on SplitCIFAR100. We use our FLOP-aware pruning technique to prune the channels and apply weight allocation with parameter $\alpha=0.1$. 
\blue{Since the impact of classifier head is rather negligible as it contains $<0.1\%$ weights and $<0.01\%$ FLOPs of the backend, we evaluate the sparsity and FLOPs for each task inside backend.} To compare the performance of different methods, we use the same random seed to generate $20$ tasks so that task sequence is exactly the same over all experiments. We conduct ESPN experiments with $20\%$, $50\%$\blue{, and $100\%$} of FLOPs constraints corresponding to ESPN-0.2, ESPN-0.5\blue{, and ESPN-1} in Table~\ref{CIFARtable}.

The results of Individual-0.2/-1 show that our FLOP-aware pruning technique can effectively reduce the computation requirements while maintaining the same level of model accuracy.  Moreover, PackNet performing better than both CPG and RMN shows the benefit of our weight allocation technique. \blue{We remark that CPG method performs pruning and fine-tuning multiple times until it finds the optimal pruning ratio, which costs significantly more time in training compared to PackNet{/ESPN}.}  ESPN-0.2/-0.5 results show that our continual learning algorithm outperforms both baselines in accuracy despite up to 80\% FLOPs reduction. \ylm{Our accuracy improvement over PackNet arises from the trainable task-specific BatchNorm weights described under Algorithm \ref{algo 1}.}
\ylm{In practice, for all CPG, RMN, PackNet, and SupSup methods, task-specific running mean and running variance inside BatchNorm layers are essential to reconstruct the same performance. Therefore, despite our method applies additional BatchNorm weights for each task, since we prune BatchNorm layers and less replay memory is needed overall compared to the other methods (except ESPN-1).}


\noindent\textbf{RotatedMNIST and PermutedMNIST.} Figure~\ref{fig:mnist} presents the RotatedMNIST and PermutedMNIST results. We run experiments on ESPN, PackNet, and individual training and report the average accuracy over $5$ trials. Note that, for fully-connected layers, we prune neurons to reduce FLOPs (rather than channels of CNN). In our experiments, we assigned neurons with a pruning parameter $\bGam$. We use our FLOP-aware pruning technique to prune neurons based on $\bGam$ and use weight allocation with parameter $\alpha=0.05$. \ylm{$\bGam$ is released after pruning.} Similar to channel-sharing in Figure \ref{overview fig}, different tasks are allowed to share the same neurons. Unlike SplitCIFAR100 experiments, all the tasks in both MNIST experiments share the same classifier head because they use the same 10 classes \blue{and FLOPs evaluation includes this head}. \blue{Therefore, \ylm{same as other zero-forgetting methods,}only binary mask is needed to restore performance.}

\begin{figure}[t]
\centering
\vspace{-30pt}
    \begin{subfigure}{0.5\textwidth}
    \centering
        \begin{tikzpicture}\centering
        \node at (0 ,0) [scale=0.33] {\includegraphics[trim={0 0 2cm 0},clip]{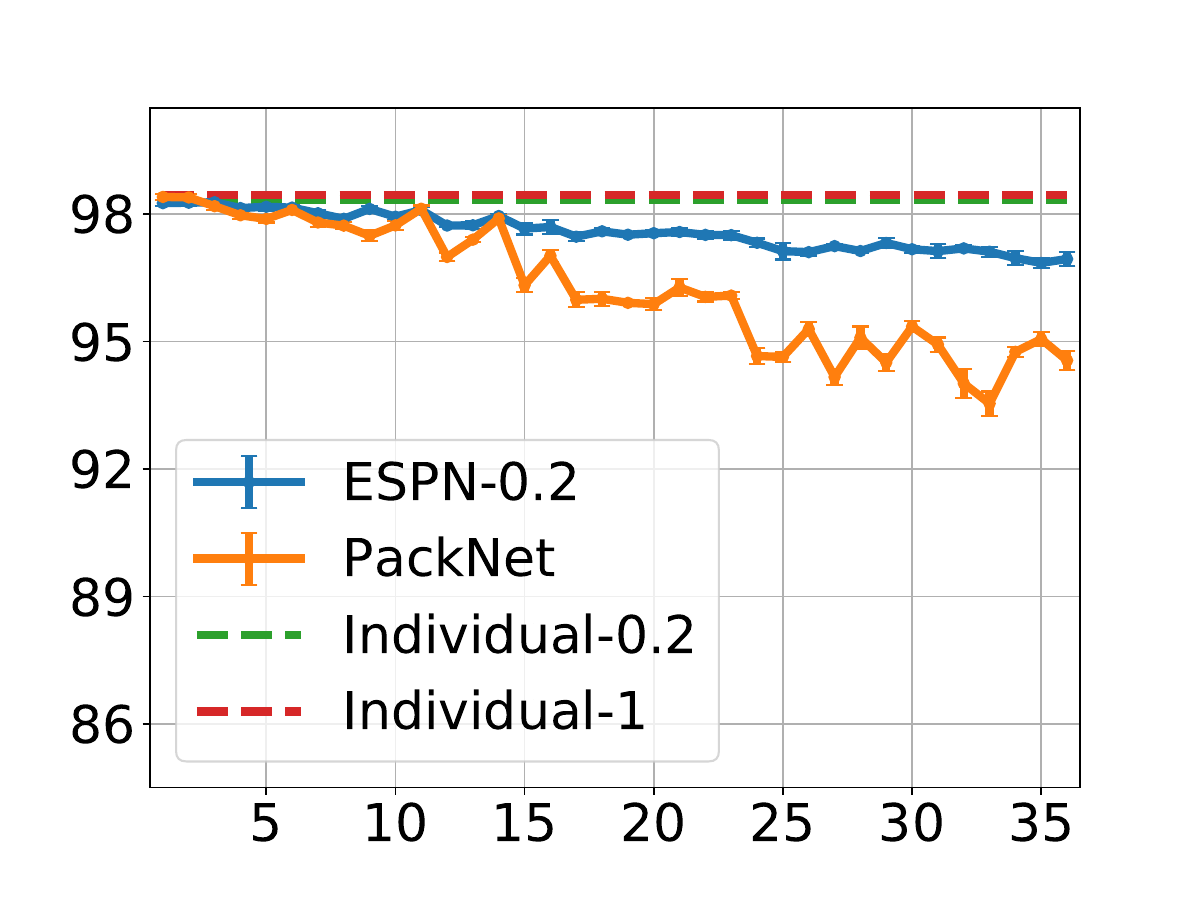}};
        \node at (0,-2.5) [scale=1.0] {\# of trained tasks};
        \node at (-3.0, 0) [scale=1.0, rotate=90] {Accuracy};
        \end{tikzpicture}\caption{{
        RotatedMNIST
		}}\label{fig:rotated}
    \end{subfigure}\hspace{-30pt}
    \begin{subfigure}{0.5\textwidth}\centering
    
        \begin{tikzpicture}\centering
        \node at (0,0) [scale=0.33] {\includegraphics[trim={2.2cm 0 2cm 0},clip]{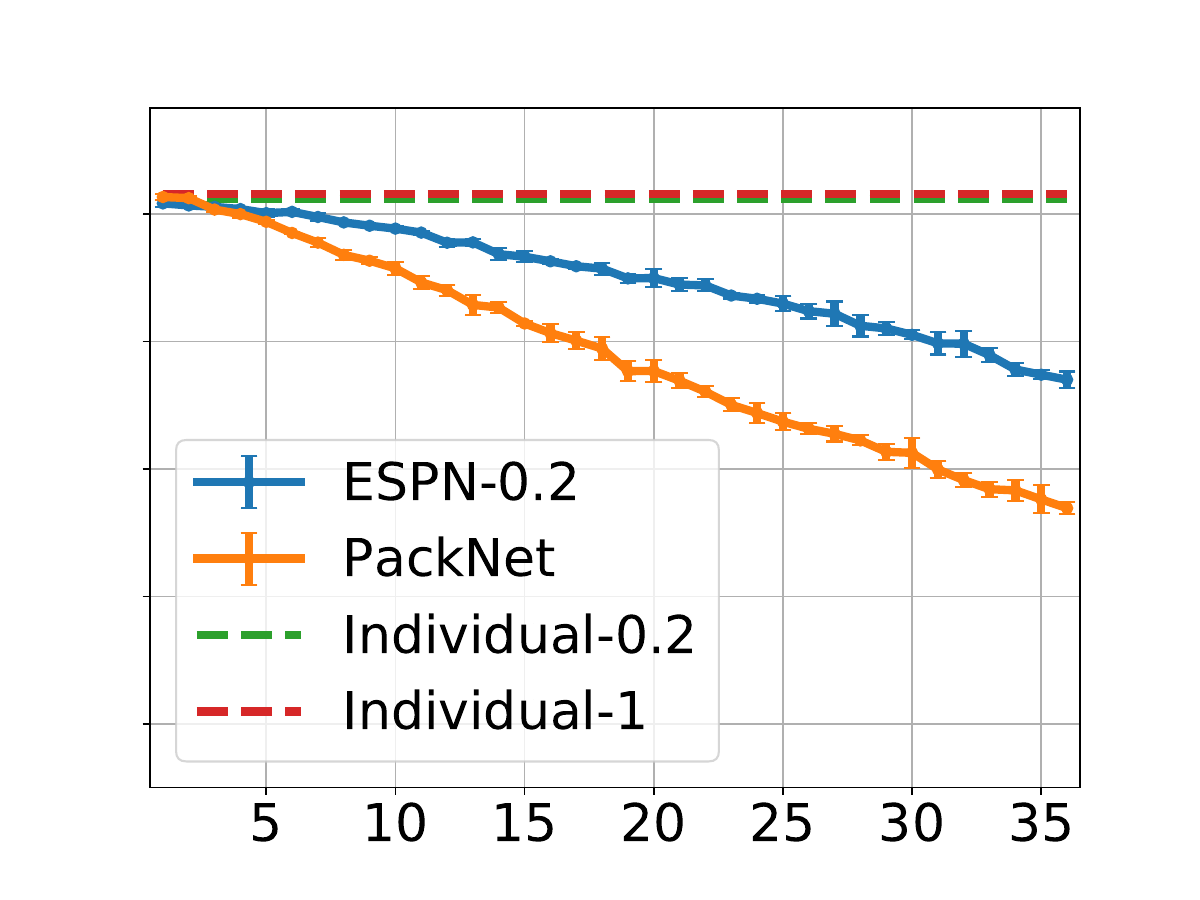}};
        \node at (0,-2.5) [scale=1.0] {\# of trained tasks};
        \end{tikzpicture}	\caption{{
        PermutedMNIST
		}}\label{fig:permuted}
    \end{subfigure} \vspace{-5pt}
    \caption{Results of experiments on RotatedMNIST and PermutedMNIST datasets with FC1024 model. Blue curves show ESPN-0.2 accuracies and Orange curves are PackNet. Both use weight allocation with $\alpha=0.05$. \blue{Both figures share the same y-axis.} (a) The test accuracy remains high because tasks in RotatedMNIST are highly correlated. (b) The test accuracy decreases over time because the pixel-permuted tasks have lower correlation and the features built for earlier tasks are not useful. 
    }\vspace{-7pt}\label{fig:mnist}
\end{figure}
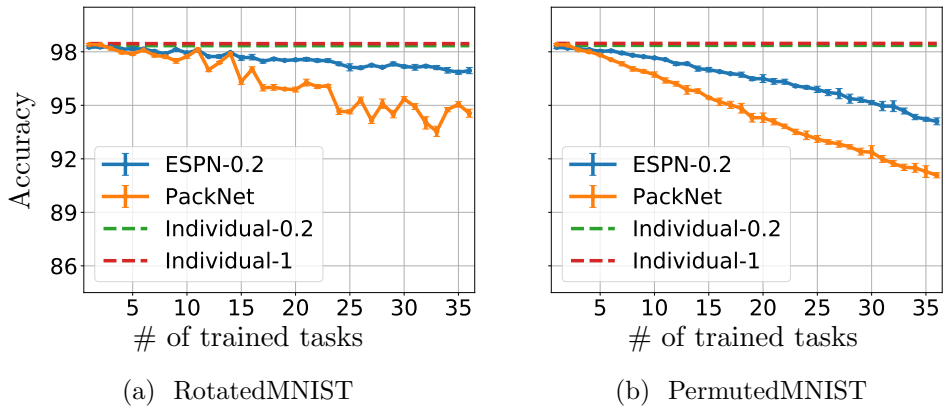

Results of RotatedMNIST and PermutedMNIST experiments are shown in Figure~\ref{fig:mnist}. Blue curves show the results of ESPN-0.2. For fair comparison, PackNet (Orange curves) use the same weight allocation parameter $\alpha=0.05$. The Green and Red dashed lines show the task-averaged accuracy of Individual-0.2/-1 baselines that train each task with separate models for $20\%/100\%$ of FLOPs constraints. \blue{In both experiments, the gap between Individual-0.2 (Green) and Individual-1 (Red) curves is negligible and it again shows the benefit of our FLOP-aware pruning technique. While PackNet performs well for the first few tasks, it degrades gradually as more tasks arrive. Thus, when there are more than a few tasks, we see that ESPN algorithm works better than PackNet despite enforcing inference-efficiency and despite using the same weight allocation method. Figure \ref{fig:rotated} shows that ESPN-0.2 on RotatedMNIST exhibits mild accuracy degradation over $36$ tasks.} While in Figure \ref{fig:permuted}, test accuracy decreases more noticeably as the tasks are added. A plausible explanation is that, because each PermutedMNIST task corresponds to a random permutation, the tasks are totally unrelated. Thus, knowledge gained from earlier tasks are not useful for training new tasks and CRL does not really help in this case. In contrast, since tasks in RotatedMNIST are semantically relevant, ESPN and PackNet both achieve higher accuracy thanks to representation reuse across tasks. Specifically, the significant accuracy gap between the Blue curves in Figures \ref{fig:rotated} and \ref{fig:permuted} (especially for larger Task IDs) demonstrate the clear benefit of CRL.


\section{Conclusion and Discussion}

To summarize, our work elucidates the benefit of continual representation learning theoretically and empirically and sheds light on the role of task ordering, diversity and sample size. We also propose a new CL algorithm to learn good representations subject to inference considerations. Extensive experimental evaluations on the proposed Efficient Sparse PackNet demonstrate its ability to achieve good accuracy as well as fast inference.


\noindent \textbf{Limitations and future directions.} 
Although we highlight the importance of task order in CRL, the task sequence is not always under our control. It would be desirable to develop adaptive learning schemes that can better identify an exploit diverse tasks and discover semantic connections across task pairs even for a predetermined task sequence. 
Another potential direction is to develop similar inference-efficient continual learning schemes for other architectures by appropriately adapting our joint weight and channel pruning strategy. An example is transformer-based models where computation and memory efficiency is particularly critical.






\acks{This work is supported in part by the National Science Foundation under grants CCF-2046816, CCF-2046293, CNS-1932254, and by the Army Research Office under grant W911NF-21-1-0312.}


\vskip 0.2in
\setlength{\bibsep}{2pt}
\bibliography{macros,refs}

\begin{thebibliography}{97}
\providecommand{\natexlab}[1]{#1}
\providecommand{\url}[1]{\texttt{#1}}
\expandafter\ifx\csname urlstyle\endcsname\relax
  \providecommand{\doi}[1]{doi: #1}\else
  \providecommand{\doi}{doi: \begingroup \urlstyle{rm}\Url}\fi

\bibitem[Aljundi et~al.(2019)Aljundi, Lin, Goujaud, and
  Bengio]{aljundi2019gradient}
Rahaf Aljundi, Min Lin, Baptiste Goujaud, and Yoshua Bengio.
\newblock Gradient based sample selection for online continual learning.
\newblock \emph{arXiv preprint arXiv:1903.08671}, 2019.

\bibitem[Alvarez and Salzmann(2016)]{alvarez2016learning}
Jose~M Alvarez and Mathieu Salzmann.
\newblock Learning the number of neurons in deep networks.
\newblock In \emph{Advances in Neural Information Processing Systems}, pages
  2270--2278, 2016.

\bibitem[Arora et~al.(2019)Arora, Khandeparkar, Khodak, Plevrakis, and
  Saunshi]{arora2019theoretical}
Sanjeev Arora, Hrishikesh Khandeparkar, Mikhail Khodak, Orestis Plevrakis, and
  Nikunj Saunshi.
\newblock A theoretical analysis of contrastive unsupervised representation
  learning.
\newblock \emph{arXiv preprint arXiv:1902.09229}, 2019.

\bibitem[Balcan et~al.(2019)Balcan, Khodak, and Talwalkar]{balcan2019provable}
Maria-Florina Balcan, Mikhail Khodak, and Ameet Talwalkar.
\newblock Provable guarantees for gradient-based meta-learning.
\newblock In \emph{International Conference on Machine Learning}, pages
  424--433, 2019.

\bibitem[Barron and Klusowski(2018)]{barron2018approximation}
Andrew~R Barron and Jason~M Klusowski.
\newblock Approximation and estimation for high-dimensional deep learning
  networks.
\newblock \emph{stat}, 1050:\penalty0 18, 2018.

\bibitem[Bartlett et~al.(2005)Bartlett, Bousquet, and
  Mendelson]{bartlett2005local}
Peter~L Bartlett, Olivier Bousquet, and Shahar Mendelson.
\newblock Local rademacher complexities.
\newblock \emph{The Annals of Statistics}, 33\penalty0 (4):\penalty0
  1497--1537, 2005.

\bibitem[Baxter(2000)]{baxter2000model}
Jonathan Baxter.
\newblock A model of inductive bias learning.
\newblock \emph{Journal of artificial intelligence research}, 12:\penalty0
  149--198, 2000.

\bibitem[Bengio et~al.(2009)Bengio, Louradour, Collobert, and
  Weston]{bengio2009curriculum}
Yoshua Bengio, J{\'e}r{\^o}me Louradour, Ronan Collobert, and Jason Weston.
\newblock Curriculum learning.
\newblock In \emph{Proceedings of the 26th annual international conference on
  machine learning}, pages 41--48, 2009.

\bibitem[Bennani et~al.(2020)Bennani, Doan, and
  Sugiyama]{bennani2020generalisation}
Mehdi~Abbana Bennani, Thang Doan, and Masashi Sugiyama.
\newblock Generalisation guarantees for continual learning with orthogonal
  gradient descent.
\newblock \emph{arXiv preprint arXiv:2006.11942}, 2020.

\bibitem[Borsos et~al.(2020)Borsos, Mutny, and Krause]{borsos2020coresets}
Zalan Borsos, Mojmir Mutny, and Andreas Krause.
\newblock Coresets via bilevel optimization for continual learning and
  streaming.
\newblock \emph{arXiv preprint arXiv:2006.03875}, 2020.

\bibitem[Buzzega et~al.(2021)Buzzega, Boschini, Porrello, and
  Calderara]{buzzega2021rethinking}
Pietro Buzzega, Matteo Boschini, Angelo Porrello, and Simone Calderara.
\newblock Rethinking experience replay: a bag of tricks for continual learning.
\newblock In \emph{2020 25th International Conference on Pattern Recognition
  (ICPR)}, pages 2180--2187. IEEE, 2021.

\bibitem[Cavallanti et~al.(2010)Cavallanti, Cesa-Bianchi, and
  Gentile]{cavallanti2010linear}
Giovanni Cavallanti, Nicolo Cesa-Bianchi, and Claudio Gentile.
\newblock Linear algorithms for online multitask classification.
\newblock \emph{The Journal of Machine Learning Research}, 11:\penalty0
  2901--2934, 2010.

\bibitem[Chang et~al.(2021)Chang, Li, Oymak, and
  Thrampoulidis]{chang2021provable}
Xiangyu Chang, Yingcong Li, Samet Oymak, and Christos Thrampoulidis.
\newblock Provable benefits of overparameterization in model compression: From
  double descent to pruning neural networks.
\newblock In \emph{Proceedings of the AAAI Conference on Artificial
  Intelligence}, volume~35, pages 6974--6983, 2021.

\bibitem[Chen et~al.(2021)Chen, Crammer, He, Roth, and Su]{chen2021weighted}
Shuxiao Chen, Koby Crammer, Hangfeng He, Dan Roth, and Weijie~J Su.
\newblock Weighted training for cross-task learning.
\newblock \emph{arXiv preprint arXiv:2105.14095}, 2021.

\bibitem[Chen and Liu(2018)]{chen2018lifelong}
Zhiyuan Chen and Bing Liu.
\newblock Lifelong machine learning.
\newblock \emph{Synthesis Lectures on Artificial Intelligence and Machine
  Learning}, 12\penalty0 (3):\penalty0 1--207, 2018.

\bibitem[Delange et~al.(2021)Delange, Aljundi, Masana, Parisot, Jia, Leonardis,
  Slabaugh, and Tuytelaars]{delange2021continual}
Matthias Delange, Rahaf Aljundi, Marc Masana, Sarah Parisot, Xu~Jia, Ales
  Leonardis, Greg Slabaugh, and Tinne Tuytelaars.
\newblock A continual learning survey: Defying forgetting in classification
  tasks.
\newblock \emph{IEEE Transactions on Pattern Analysis and Machine
  Intelligence}, 2021.

\bibitem[Denevi et~al.(2019)Denevi, Stamos, Ciliberto, and
  Pontil]{denevi2019online}
Giulia Denevi, Dimitris Stamos, Carlo Ciliberto, and Massimiliano Pontil.
\newblock Online-within-online meta-learning.
\newblock In \emph{ADVANCES IN NEURAL INFORMATION PROCESSING SYSTEMS 32 (NIPS
  2019)}, volume~32, pages 1--11. Neural Information Processing Systems
  (NeurIPS 2019), 2019.

\bibitem[Doan et~al.(2021)Doan, Bennani, Mazoure, Rabusseau, and
  Alquier]{doan2021theoretical}
Thang Doan, Mehdi~Abbana Bennani, Bogdan Mazoure, Guillaume Rabusseau, and
  Pierre Alquier.
\newblock A theoretical analysis of catastrophic forgetting through the ntk
  overlap matrix.
\newblock In \emph{International Conference on Artificial Intelligence and
  Statistics}, pages 1072--1080. PMLR, 2021.

\bibitem[Du et~al.(2021)Du, Hu, Kakade, Lee, and Lei]{du2020few}
Simon~S Du, Wei Hu, Sham~M Kakade, Jason~D Lee, and Qi~Lei.
\newblock Few-shot learning via learning the representation, provably.
\newblock \emph{International Conference on Learning Representations}, 2021.

\bibitem[Eitz et~al.(2012)Eitz, Hays, and Alexa]{eitz2012humans}
Mathias Eitz, James Hays, and Marc Alexa.
\newblock How do humans sketch objects?
\newblock \emph{ACM Transactions on graphics (TOG)}, 31\penalty0 (4):\penalty0
  1--10, 2012.

\bibitem[Farajtabar et~al.(2020)Farajtabar, Azizan, Mott, and
  Li]{farajtabar2020orthogonal}
Mehrdad Farajtabar, Navid Azizan, Alex Mott, and Ang Li.
\newblock Orthogonal gradient descent for continual learning.
\newblock In \emph{International Conference on Artificial Intelligence and
  Statistics}, pages 3762--3773. PMLR, 2020.

\bibitem[Fernando et~al.(2017)Fernando, Banarse, Blundell, Zwols, Ha, Rusu,
  Pritzel, and Wierstra]{fernando2017pathnet}
Chrisantha Fernando, Dylan Banarse, Charles Blundell, Yori Zwols, David Ha,
  Andrei~A Rusu, Alexander Pritzel, and Daan Wierstra.
\newblock Pathnet: Evolution channels gradient descent in super neural
  networks.
\newblock \emph{arXiv preprint arXiv:1701.08734}, 2017.

\bibitem[Finn et~al.(2017)Finn, Abbeel, and Levine]{finn2017model}
Chelsea Finn, Pieter Abbeel, and Sergey Levine.
\newblock Model-agnostic meta-learning for fast adaptation of deep networks.
\newblock In \emph{International Conference on Machine Learning}, pages
  1126--1135. PMLR, 2017.

\bibitem[Frankle and Carbin(2018)]{frankle2018lottery}
Jonathan Frankle and Michael Carbin.
\newblock The lottery ticket hypothesis: Finding sparse, trainable neural
  networks.
\newblock \emph{arXiv preprint arXiv:1803.03635}, 2018.

\bibitem[Garg and Liang(2020)]{garg2020functional}
Siddhant Garg and Yingyu Liang.
\newblock Functional regularization for representation learning: A unified
  theoretical perspective.
\newblock \emph{arXiv preprint arXiv:2008.02447}, 2020.

\bibitem[Graves(2013)]{graves2013generating}
Alex Graves.
\newblock Generating sequences with recurrent neural networks.
\newblock \emph{arXiv preprint arXiv:1308.0850}, 2013.

\bibitem[Gulluk et~al.(2021)Gulluk, Sun, Oymak, and Fazel]{gulluk2021sample}
Halil~Ibrahim Gulluk, Yue Sun, Samet Oymak, and Maryam Fazel.
\newblock Sample efficient subspace-based representations for nonlinear
  meta-learning.
\newblock In \emph{ICASSP 2021-2021 IEEE International Conference on Acoustics,
  Speech and Signal Processing (ICASSP)}, pages 3685--3689. IEEE, 2021.

\bibitem[Han et~al.(2015{\natexlab{a}})Han, Mao, and Dally]{han2015deep}
Song Han, Huizi Mao, and William~J Dally.
\newblock Deep compression: Compressing deep neural networks with pruning,
  trained quantization and huffman coding.
\newblock \emph{arXiv preprint arXiv:1510.00149}, 2015{\natexlab{a}}.

\bibitem[Han et~al.(2015{\natexlab{b}})Han, Pool, Tran, and
  Dally]{han2015learning}
Song Han, Jeff Pool, John Tran, and William Dally.
\newblock Learning both weights and connections for efficient neural network.
\newblock In \emph{Advances in Neural Information Processing Systems}, pages
  1135--1143, 2015{\natexlab{b}}.

\bibitem[Han et~al.(2016)Han, Liu, Mao, Pu, Pedram, Horowitz, and
  Dally]{han2016eie}
Song Han, Xingyu Liu, Huizi Mao, Jing Pu, Ardavan Pedram, Mark~A Horowitz, and
  William~J Dally.
\newblock Eie: Efficient inference engine on compressed deep neural network.
\newblock \emph{ACM SIGARCH Computer Architecture News}, 44\penalty0
  (3):\penalty0 243--254, 2016.

\bibitem[Hanneke and Kpotufe(2020)]{hanneke2020no}
Steve Hanneke and Samory Kpotufe.
\newblock A no-free-lunch theorem for multitask learning.
\newblock \emph{arXiv preprint arXiv:2006.15785}, 2020.

\bibitem[Hassibi and Stork(1993)]{hassibi1993second}
Babak Hassibi and David~G Stork.
\newblock Second order derivatives for network pruning: Optimal brain surgeon.
\newblock In \emph{Advances in neural information processing systems}, pages
  164--171, 1993.

\bibitem[He et~al.(2017)He, Zhang, and Sun]{he2017channel}
Yihui He, Xiangyu Zhang, and Jian Sun.
\newblock Channel pruning for accelerating very deep neural networks.
\newblock In \emph{Proceedings of the IEEE international conference on computer
  vision}, pages 1389--1397, 2017.

\bibitem[Hung et~al.(2019)Hung, Tu, Wu, Chen, Chan, and
  Chen]{hung2019compacting}
Steven~CY Hung, Cheng-Hao Tu, Cheng-En Wu, Chien-Hung Chen, Yi-Ming Chan, and
  Chu-Song Chen.
\newblock Compacting, picking and growing for unforgetting continual learning.
\newblock \emph{arXiv preprint arXiv:1910.06562}, 2019.

\bibitem[Jung et~al.(2020)Jung, Ahn, Cha, and Moon]{jung2020continual}
Sangwon Jung, Hongjoon Ahn, Sungmin Cha, and Taesup Moon.
\newblock Continual learning with node-importance based adaptive group sparse
  regularization.
\newblock \emph{arXiv preprint arXiv:2003.13726}, 2020.

\bibitem[Kaushik et~al.(2021)Kaushik, Gain, Kortylewski, and
  Yuille]{kaushik2021understanding}
Prakhar Kaushik, Alex Gain, Adam Kortylewski, and Alan Yuille.
\newblock Understanding catastrophic forgetting and remembering in continual
  learning with optimal relevance mapping.
\newblock \emph{arXiv preprint arXiv:2102.11343}, 2021.

\bibitem[Khodak et~al.(2019)Khodak, Balcan, and Talwalkar]{khodak2019adaptive}
Mikhail Khodak, Maria-Florina~F Balcan, and Ameet~S Talwalkar.
\newblock Adaptive gradient-based meta-learning methods.
\newblock In \emph{Advances in Neural Information Processing Systems}, pages
  5917--5928, 2019.

\bibitem[Kingma and Ba(2014)]{kingma2014adam}
Diederik~P Kingma and Jimmy Ba.
\newblock Adam: A method for stochastic optimization.
\newblock \emph{arXiv preprint arXiv:1412.6980}, 2014.

\bibitem[Kirkpatrick et~al.(2017)Kirkpatrick, Pascanu, Rabinowitz, Veness,
  Desjardins, Rusu, Milan, Quan, Ramalho, Grabska-Barwinska,
  et~al.]{kirkpatrick2017overcoming}
James Kirkpatrick, Razvan Pascanu, Neil Rabinowitz, Joel Veness, Guillaume
  Desjardins, Andrei~A Rusu, Kieran Milan, John Quan, Tiago Ramalho, Agnieszka
  Grabska-Barwinska, et~al.
\newblock Overcoming catastrophic forgetting in neural networks.
\newblock \emph{Proceedings of the national academy of sciences}, 114\penalty0
  (13):\penalty0 3521--3526, 2017.

\bibitem[Kong et~al.(2020)Kong, Somani, Song, Kakade, and Oh]{kong2020meta}
Weihao Kong, Raghav Somani, Zhao Song, Sham Kakade, and Sewoong Oh.
\newblock Meta-learning for mixed linear regression.
\newblock In \emph{International Conference on Machine Learning}, pages
  5394--5404. PMLR, 2020.

\bibitem[Krause et~al.(2013)Krause, Stark, Deng, and Fei-Fei]{krause20133d}
Jonathan Krause, Michael Stark, Jia Deng, and Li~Fei-Fei.
\newblock 3d object representations for fine-grained categorization.
\newblock In \emph{Proceedings of the IEEE international conference on computer
  vision workshops}, pages 554--561, 2013.

\bibitem[Krizhevsky et~al.(2009)Krizhevsky, Hinton,
  et~al.]{krizhevsky2009learning}
Alex Krizhevsky, Geoffrey Hinton, et~al.
\newblock Learning multiple layers of features from tiny images.
\newblock \emph{technical report}, 2009.

\bibitem[Krizhevsky et~al.(2012)Krizhevsky, Sutskever, and
  Hinton]{krizhevsky2012imagenet}
Alex Krizhevsky, Ilya Sutskever, and Geoffrey~E Hinton.
\newblock Imagenet classification with deep convolutional neural networks.
\newblock \emph{Advances in neural information processing systems},
  25:\penalty0 1097--1105, 2012.

\bibitem[Lebedev and Lempitsky(2016)]{lebedev2016fast}
Vadim Lebedev and Victor Lempitsky.
\newblock Fast convnets using group-wise brain damage.
\newblock In \emph{Proceedings of the IEEE Conference on Computer Vision and
  Pattern Recognition}, pages 2554--2564, 2016.

\bibitem[LeCun et~al.(1990)LeCun, Denker, and Solla]{lecun1990optimal}
Yann LeCun, John~S Denker, and Sara~A Solla.
\newblock Optimal brain damage.
\newblock In \emph{Advances in neural information processing systems}, pages
  598--605, 1990.

\bibitem[Lee et~al.(2021)Lee, Goldt, and Saxe]{lee2021continual}
Sebastian Lee, Sebastian Goldt, and Andrew Saxe.
\newblock Continual learning in the teacher-student setup: Impact of task
  similarity.
\newblock In \emph{International Conference on Machine Learning}, pages
  6109--6119. PMLR, 2021.

\bibitem[Li and Hoiem(2017)]{li2017learning}
Zhizhong Li and Derek Hoiem.
\newblock Learning without forgetting.
\newblock \emph{IEEE transactions on pattern analysis and machine
  intelligence}, 40\penalty0 (12):\penalty0 2935--2947, 2017.

\bibitem[Liu et~al.(2017)Liu, Li, Shen, Huang, Yan, and Zhang]{liu2017learning}
Zhuang Liu, Jianguo Li, Zhiqiang Shen, Gao Huang, Shoumeng Yan, and Changshui
  Zhang.
\newblock Learning efficient convolutional networks through network slimming.
\newblock In \emph{Proceedings of the IEEE international conference on computer
  vision}, pages 2736--2744, 2017.

\bibitem[Lopez-Paz and Ranzato(2017)]{lopez2017gradient}
David Lopez-Paz and Marc'Aurelio Ranzato.
\newblock Gradient episodic memory for continual learning.
\newblock \emph{Advances in neural information processing systems},
  30:\penalty0 6467--6476, 2017.

\bibitem[Loshchilov and Hutter(2016)]{loshchilov2016sgdr}
Ilya Loshchilov and Frank Hutter.
\newblock Sgdr: Stochastic gradient descent with warm restarts.
\newblock \emph{arXiv preprint arXiv:1608.03983}, 2016.

\bibitem[Lounici et~al.(2011)Lounici, Pontil, Van De~Geer, and
  Tsybakov]{lounici2011oracle}
Karim Lounici, Massimiliano Pontil, Sara Van De~Geer, and Alexandre~B Tsybakov.
\newblock Oracle inequalities and optimal inference under group sparsity.
\newblock \emph{The annals of statistics}, 39\penalty0 (4):\penalty0
  2164--2204, 2011.

\bibitem[Lu et~al.(2021)Lu, Huang, and Du]{lu2021power}
Rui Lu, Gao Huang, and Simon~S Du.
\newblock On the power of multitask representation learning in linear mdp.
\newblock \emph{arXiv preprint arXiv:2106.08053}, 2021.

\bibitem[Lubana et~al.(2021)Lubana, Trivedi, Koutra, and
  Dick]{lubana2021quadratic}
Ekdeep~Singh Lubana, Puja Trivedi, Danai Koutra, and Robert~P Dick.
\newblock How do quadratic regularizers prevent catastrophic forgetting: The
  role of interpolation.
\newblock \emph{arXiv preprint arXiv:2102.02805}, 2021.

\bibitem[Malach et~al.(2020)Malach, Yehudai, Shalev-Shwartz, and
  Shamir]{malach2020proving}
Eran Malach, Gilad Yehudai, Shai Shalev-Shwartz, and Ohad Shamir.
\newblock Proving the lottery ticket hypothesis: Pruning is all you need.
\newblock \emph{arXiv preprint arXiv:2002.00585}, 2020.

\bibitem[Mallya and Lazebnik(2018)]{mallya2018packnet}
Arun Mallya and Svetlana Lazebnik.
\newblock Packnet: Adding multiple tasks to a single network by iterative
  pruning.
\newblock In \emph{Proceedings of the IEEE conference on Computer Vision and
  Pattern Recognition}, pages 7765--7773, 2018.

\bibitem[Mallya et~al.(2018)Mallya, Davis, and Lazebnik]{mallya2018piggyback}
Arun Mallya, Dillon Davis, and Svetlana Lazebnik.
\newblock Piggyback: Adapting a single network to multiple tasks by learning to
  mask weights.
\newblock In \emph{Proceedings of the European Conference on Computer Vision
  (ECCV)}, pages 67--82, 2018.

\bibitem[Maurer et~al.(2016)Maurer, Pontil, and
  Romera-Paredes]{maurer2016benefit}
Andreas Maurer, Massimiliano Pontil, and Bernardino Romera-Paredes.
\newblock The benefit of multitask representation learning.
\newblock \emph{Journal of Machine Learning Research}, 17\penalty0
  (81):\penalty0 1--32, 2016.

\bibitem[McCloskey and Cohen(1989)]{mccloskey1989catastrophic}
Michael McCloskey and Neal~J Cohen.
\newblock Catastrophic interference in connectionist networks: The sequential
  learning problem.
\newblock In \emph{Psychology of learning and motivation}, volume~24, pages
  109--165. Elsevier, 1989.

\bibitem[Mendelson(2003)]{mendelson2003few}
Shahar Mendelson.
\newblock A few notes on statistical learning theory.
\newblock In \emph{Advanced lectures on machine learning}, pages 1--40.
  Springer, 2003.

\bibitem[Mirzadeh et~al.(2020)Mirzadeh, Farajtabar, Pascanu, and
  Ghasemzadeh]{mirzadeh2020understanding}
Seyed~Iman Mirzadeh, Mehrdad Farajtabar, Razvan Pascanu, and Hassan
  Ghasemzadeh.
\newblock Understanding the role of training regimes in continual learning.
\newblock \emph{arXiv preprint arXiv:2006.06958}, 2020.

\bibitem[Nilsback and Zisserman(2008)]{nilsback2008automated}
Maria-Elena Nilsback and Andrew Zisserman.
\newblock Automated flower classification over a large number of classes.
\newblock In \emph{2008 Sixth Indian Conference on Computer Vision, Graphics \&
  Image Processing}, pages 722--729. IEEE, 2008.

\bibitem[Oymak et~al.(2021)Oymak, Li, and
  Soltanolkotabi]{oymak2021generalization}
Samet Oymak, Mingchen Li, and Mahdi Soltanolkotabi.
\newblock Generalization guarantees for neural architecture search with
  train-validation split.
\newblock In \emph{International Conference on Machine Learning}, pages
  8291--8301. PMLR, 2021.

\bibitem[Parisi et~al.(2019)Parisi, Kemker, Part, Kanan, and
  Wermter]{parisi2019continual}
German~I Parisi, Ronald Kemker, Jose~L Part, Christopher Kanan, and Stefan
  Wermter.
\newblock Continual lifelong learning with neural networks: A review.
\newblock \emph{Neural Networks}, 113:\penalty0 54--71, 2019.

\bibitem[Pensia et~al.(2020)Pensia, Rajput, Nagle, Vishwakarma, and
  Papailiopoulos]{pensia2020optimal}
Ankit Pensia, Shashank Rajput, Alliot Nagle, Harit Vishwakarma, and Dimitris
  Papailiopoulos.
\newblock Optimal lottery tickets via subsetsum: Logarithmic
  over-parameterization is sufficient.
\newblock \emph{arXiv preprint arXiv:2006.07990}, 2020.

\bibitem[Pf{\"u}lb and Gepperth(2019)]{pfulb2019comprehensive}
Benedikt Pf{\"u}lb and Alexander Gepperth.
\newblock A comprehensive, application-oriented study of catastrophic
  forgetting in dnns.
\newblock \emph{arXiv preprint arXiv:1905.08101}, 2019.

\bibitem[Pontil and Maurer(2013)]{pontil2013excess}
Massimiliano Pontil and Andreas Maurer.
\newblock Excess risk bounds for multitask learning with trace norm
  regularization.
\newblock In \emph{Conference on Learning Theory}, pages 55--76. PMLR, 2013.

\bibitem[Qin et~al.(2022)Qin, Menara, Oymak, Ching, and
  Pasqualetti]{qin2022non}
Yuzhen Qin, Tommaso Menara, Samet Oymak, ShiNung Ching, and Fabio Pasqualetti.
\newblock Non-stationary representation learning in sequential linear bandits.
\newblock \emph{arXiv preprint arXiv:2201.04805}, 2022.

\bibitem[Ramanujan et~al.(2020)Ramanujan, Wortsman, Kembhavi, Farhadi, and
  Rastegari]{ramanujan2020s}
Vivek Ramanujan, Mitchell Wortsman, Aniruddha Kembhavi, Ali Farhadi, and
  Mohammad Rastegari.
\newblock What's hidden in a randomly weighted neural network?
\newblock In \emph{Proceedings of the IEEE/CVF Conference on Computer Vision
  and Pattern Recognition}, pages 11893--11902, 2020.

\bibitem[Ramesh and Chaudhari(2021)]{ramesh2021model}
Rahul Ramesh and Pratik Chaudhari.
\newblock Model zoo: A growing brain that learns continually.
\newblock In \emph{NeurIPS 2021 Workshop on Distribution Shifts: Connecting
  Methods and Applications}, 2021.

\bibitem[Rebuffi et~al.(2017)Rebuffi, Kolesnikov, Sperl, and
  Lampert]{rebuffi2017icarl}
Sylvestre-Alvise Rebuffi, Alexander Kolesnikov, Georg Sperl, and Christoph~H
  Lampert.
\newblock icarl: Incremental classifier and representation learning.
\newblock In \emph{Proceedings of the IEEE conference on Computer Vision and
  Pattern Recognition}, pages 2001--2010, 2017.

\bibitem[Rolnick et~al.(2018)Rolnick, Ahuja, Schwarz, Lillicrap, and
  Wayne]{rolnick2018experience}
David Rolnick, Arun Ahuja, Jonathan Schwarz, Timothy~P Lillicrap, and Greg
  Wayne.
\newblock Experience replay for continual learning.
\newblock \emph{arXiv preprint arXiv:1811.11682}, 2018.

\bibitem[Rusu et~al.(2016)Rusu, Rabinowitz, Desjardins, Soyer, Kirkpatrick,
  Kavukcuoglu, Pascanu, and Hadsell]{rusu2016progressive}
Andrei~A Rusu, Neil~C Rabinowitz, Guillaume Desjardins, Hubert Soyer, James
  Kirkpatrick, Koray Kavukcuoglu, Razvan Pascanu, and Raia Hadsell.
\newblock Progressive neural networks.
\newblock \emph{arXiv preprint arXiv:1606.04671}, 2016.

\bibitem[Saleh and Elgammal(2015)]{saleh2015large}
Babak Saleh and Ahmed Elgammal.
\newblock Large-scale classification of fine-art paintings: Learning the right
  metric on the right feature.
\newblock \emph{arXiv preprint arXiv:1505.00855}, 2015.

\bibitem[Srebro et~al.(2010)Srebro, Sridharan, and
  Tewari]{srebro2010smoothness}
Nathan Srebro, Karthik Sridharan, and Ambuj Tewari.
\newblock Smoothness, low noise and fast rates.
\newblock \emph{Advances in neural information processing systems}, 23, 2010.

\bibitem[Sun et~al.(2021)Sun, Narang, Gulluk, Oymak, and Fazel]{sun2021towards}
Yue Sun, Adhyyan Narang, Ibrahim Gulluk, Samet Oymak, and Maryam Fazel.
\newblock Towards sample-efficient overparameterized meta-learning.
\newblock \emph{Advances in Neural Information Processing Systems}, 34, 2021.

\bibitem[Sze et~al.(2017)Sze, Chen, Yang, and Emer]{sze2017efficient}
Vivienne Sze, Yu-Hsin Chen, Tien-Ju Yang, and Joel~S Emer.
\newblock Efficient processing of deep neural networks: A tutorial and survey.
\newblock \emph{Proceedings of the IEEE}, 105\penalty0 (12):\penalty0
  2295--2329, 2017.

\bibitem[Thrun(1998)]{thrun1998lifelong}
Sebastian Thrun.
\newblock Lifelong learning algorithms.
\newblock In \emph{Learning to learn}, pages 181--209. Springer, 1998.

\bibitem[Tripuraneni et~al.(2020{\natexlab{a}})Tripuraneni, Jin, and
  Jordan]{tripuraneni2020provable}
Nilesh Tripuraneni, Chi Jin, and Michael~I Jordan.
\newblock Provable meta-learning of linear representations.
\newblock \emph{arXiv preprint arXiv:2002.11684}, 2020{\natexlab{a}}.

\bibitem[Tripuraneni et~al.(2020{\natexlab{b}})Tripuraneni, Jordan, and
  Jin]{tripuraneni2020theory}
Nilesh Tripuraneni, Michael~I Jordan, and Chi Jin.
\newblock On the theory of transfer learning: The importance of task diversity.
\newblock \emph{arXiv preprint arXiv:2006.11650}, 2020{\natexlab{b}}.

\bibitem[Tu et~al.(2020)Tu, Wu, and Chen]{tu2020extending}
Cheng-Hao Tu, Cheng-En Wu, and Chu-Song Chen.
\newblock Extending conditional convolution structures for enhancing
  multitasking continual learning.
\newblock In \emph{2020 Asia-Pacific Signal and Information Processing
  Association Annual Summit and Conference (APSIPA ASC)}, pages 1605--1610.
  IEEE, 2020.

\bibitem[Van~de Ven and Tolias(2019)]{van2019three}
Gido~M Van~de Ven and Andreas~S Tolias.
\newblock Three scenarios for continual learning.
\newblock \emph{arXiv preprint arXiv:1904.07734}, 2019.

\bibitem[Vapnik and Chervonenkis(2015)]{vapnik2015uniform}
Vladimir~N Vapnik and A~Ya Chervonenkis.
\newblock On the uniform convergence of relative frequencies of events to their
  probabilities.
\newblock In \emph{Measures of complexity}, pages 11--30. Springer, 2015.

\bibitem[Wah et~al.(2011)Wah, Branson, Welinder, Perona, and
  Belongie]{wah2011caltech}
Catherine Wah, Steve Branson, Peter Welinder, Pietro Perona, and Serge
  Belongie.
\newblock The caltech-ucsd birds-200-2011 dataset.
\newblock 2011.

\bibitem[Wang et~al.(2016)Wang, Kolar, and Srebro]{wang2016distributed}
Jialei Wang, Mladen Kolar, and Nathan Srebro.
\newblock Distributed multi-task learning with shared representation.
\newblock \emph{arXiv preprint arXiv:1603.02185}, 2016.

\bibitem[Wen et~al.(2016)Wen, Wu, Wang, Chen, and Li]{wen2016learning}
Wei Wen, Chunpeng Wu, Yandan Wang, Yiran Chen, and Hai Li.
\newblock Learning structured sparsity in deep neural networks.
\newblock \emph{Advances in neural information processing systems},
  29:\penalty0 2074--2082, 2016.

\bibitem[Wortsman et~al.(2019)Wortsman, Farhadi, and
  Rastegari]{Wortsman2019DiscoveringNW}
Mitchell Wortsman, Ali Farhadi, and Mohammad Rastegari.
\newblock Discovering neural wirings.
\newblock \emph{ArXiv}, abs/1906.00586, 2019.

\bibitem[Wortsman et~al.(2020)Wortsman, Ramanujan, Liu, Kembhavi, Rastegari,
  Yosinski, and Farhadi]{wortsman2020supermasks}
Mitchell Wortsman, Vivek Ramanujan, Rosanne Liu, Aniruddha Kembhavi, Mohammad
  Rastegari, Jason Yosinski, and Ali Farhadi.
\newblock Supermasks in superposition.
\newblock \emph{arXiv preprint arXiv:2006.14769}, 2020.

\bibitem[Wu et~al.(2020)Wu, Zhang, and R{\'e}]{wu2020understanding}
Sen Wu, Hongyang~R Zhang, and Christopher R{\'e}.
\newblock Understanding and improving information transfer in multi-task
  learning.
\newblock \emph{arXiv preprint arXiv:2005.00944}, 2020.

\bibitem[Xu and Tewari(2021{\natexlab{a}})]{xu2021representation}
Ziping Xu and Ambuj Tewari.
\newblock Representation learning beyond linear prediction functions.
\newblock \emph{arXiv preprint arXiv:2105.14989}, 2021{\natexlab{a}}.

\bibitem[Xu and Tewari(2021{\natexlab{b}})]{xu2021statistical}
Ziping Xu and Ambuj Tewari.
\newblock On the statistical benefits of curriculum learning.
\newblock \emph{arXiv preprint arXiv:2111.07126}, 2021{\natexlab{b}}.

\bibitem[Ye et~al.(2018)Ye, Lu, Lin, and Wang]{ye2018rethinking}
Jianbo Ye, Xin Lu, Zhe Lin, and James~Z Wang.
\newblock Rethinking the smaller-norm-less-informative assumption in channel
  pruning of convolution layers.
\newblock \emph{arXiv preprint arXiv:1802.00124}, 2018.

\bibitem[Yin et~al.(2020)Yin, Farajtabar, Li, Levine, and
  Mott]{yin2020optimization}
Dong Yin, Mehrdad Farajtabar, Ang Li, Nir Levine, and Alex Mott.
\newblock Optimization and generalization of regularization-based continual
  learning: a loss approximation viewpoint.
\newblock \emph{arXiv preprint arXiv:2006.10974}, 2020.

\bibitem[Yoon et~al.(2017)Yoon, Yang, Lee, and Hwang]{yoon2017lifelong}
Jaehong Yoon, Eunho Yang, Jeongtae Lee, and Sung~Ju Hwang.
\newblock Lifelong learning with dynamically expandable networks.
\newblock \emph{arXiv preprint arXiv:1708.01547}, 2017.

\bibitem[Zenke et~al.(2017)Zenke, Poole, and Ganguli]{zenke2017continual}
Friedemann Zenke, Ben Poole, and Surya Ganguli.
\newblock Continual learning through synaptic intelligence.
\newblock In \emph{International Conference on Machine Learning}, pages
  3987--3995. PMLR, 2017.

\bibitem[Zhou et~al.(2016)Zhou, Alvarez, and Porikli]{zhou2016less}
Hao Zhou, Jose~M Alvarez, and Fatih Porikli.
\newblock Less is more: Towards compact cnns.
\newblock In \emph{European Conference on Computer Vision}, pages 662--677.
  Springer, 2016.

\bibitem[Zhou et~al.(2019)Zhou, Lan, Liu, and Yosinski]{zhou2019deconstructing}
Hattie Zhou, Janice Lan, Rosanne Liu, and Jason Yosinski.
\newblock Deconstructing lottery tickets: Zeros, signs, and the supermask.
\newblock \emph{arXiv preprint arXiv:1905.01067}, 2019.

\bibitem[Zhuang et~al.(2020)Zhuang, Zhang, Huang, Zeng, Shuang, and
  Li]{zhuang2020neuron}
Tao Zhuang, Zhixuan Zhang, Yuheng Huang, Xiaoyi Zeng, Kai Shuang, and Xiang Li.
\newblock Neuron-level structured pruning using polarization regularizer.
\newblock In \emph{NeurIPS}, 2020.

\end{thebibliography}

\appendix
\red{
\section*{Organization of the Appendix}
The material in the appendix is organized as follows.
\begin{enumerate}
\item Appendix~\ref{pruning sec} demonstrates the benefits of our FLOP-aware pruning and weight allocation strategies.
\item Appendix~\ref{app B} provides the proofs of our Theorem \ref{cl thm} and Corollary~\ref{cl thm3}. 
\item Appendix~\ref{app:proof seq thm} provides the proof of our Theorem \ref{seq thm}. 
\item Appendix~\ref{app:imagenet} provides the implementation details of the experiments shown in Fig.~\ref{fig:diversity}.
\end{enumerate}
}

\section{Experimental Evaluations of Channel Pruning and Weight Allocation Techniques}\label{pruning sec}

\noi\textbf{Empirical benefits of our proposed FLOP-aware pruning.} \blue{Both SplitCIFAR100 (Table~\ref{CIFARtable}) and MNIST (Fig.~\ref{fig:mnist}) experiments show that our pruning algorithm performs well (by comparison between Individual-0.2/-1)}\ylm{Table~\ref{CIFARtable} shows that our algorithm performs well in our continual learning manner}, but missing additional experimental comparisons to alternative channel-pruning techniques. In this section, we evaluate our FLOP-aware pruning algorithm separately without CL and compare our results to $\ell_1$-based~(\cite{liu2017learning}) and Polarization-based~(\cite{zhuang2020neuron}) channel pruning methods. For all the experiments, we use the same ResNet18 architecture and hyperparameters as Table~\ref{CIFARtable} and train over CIFAR10 dataset. Results are presented in Fig.~\ref{fig:CIFAR10_acc}\&\ref{fig:CIFAR10_ratio} (averaged from 5 trials). We implement these three methods over different FLOPs constraints and track the fraction of nonzero channels after pruning. Fig.~\ref{fig:CIFAR10_acc} shows test accuracy after pruning. We observe our FLOP-aware pruning method (shown in the Blue curve) does well in maintaining high accuracy compared to the other methods. The benefit of our approach is most visible when the FLOPs constraint is more aggressive (e.g., around 1\% FLOPs). The Green curve shows the Polarization-based method and demonstrates that it does not execute the pruning task successfully with less than 50\% FLOPs constraint. We found that this is because below a certain threshold, Polarization tends to prune a whole layer resulting in disconnectivity within the network. Fig.~\ref{fig:CIFAR10_ratio} shows the fraction of nonzero channels that remained after pruning, and our method succeeds in keeping more channels while pruning the same number of FLOPs. Here, we emphasize keeping more channels is a measure of connectivity within the network (as we wish to avoid very sparse layers). We also observe that $\ell_1$ works well even for small FLOPs; however, it uses the same $\lambda_l$ penalization for all layers thus it is not FLOP-aware and performs uniformly worse than our FLOP-aware algorithm.

\begin{figure*}[t]
\vspace{-10pt}
\centering
\begin{subfigure}[t]{.33\textwidth}
  \centering
  \begin{tikzpicture}\hspace{-0pt}
        \node at (0,0) [scale=1.] {\includegraphics[width=\linewidth]{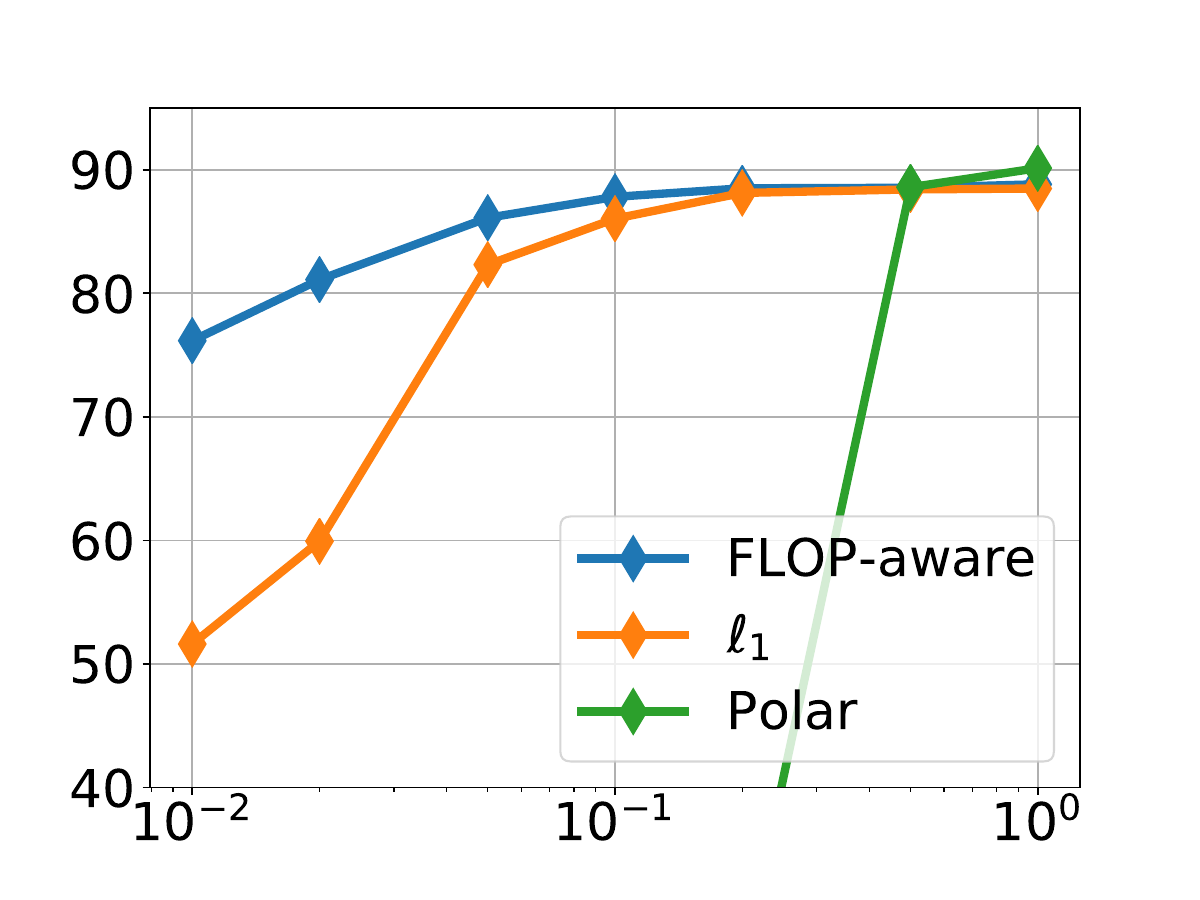}};
        \node at (0,-1.9) [scale=0.8] {{FLOPs constraint $\gamma$}};
        \node at (-2.4,0) [scale=0.8,rotate=90] {Accuracy};
    \end{tikzpicture}
        \centering
  \caption{Test accuracy}
	\hspace{-0pt}\label{fig:CIFAR10_acc}
\end{subfigure}\hspace{0pt}\begin{subfigure}[t]{.33\textwidth}
  \centering
  \begin{tikzpicture}\hspace{-0pt}
        \node at (-0,0) [scale=1.] {\includegraphics[width=\linewidth]{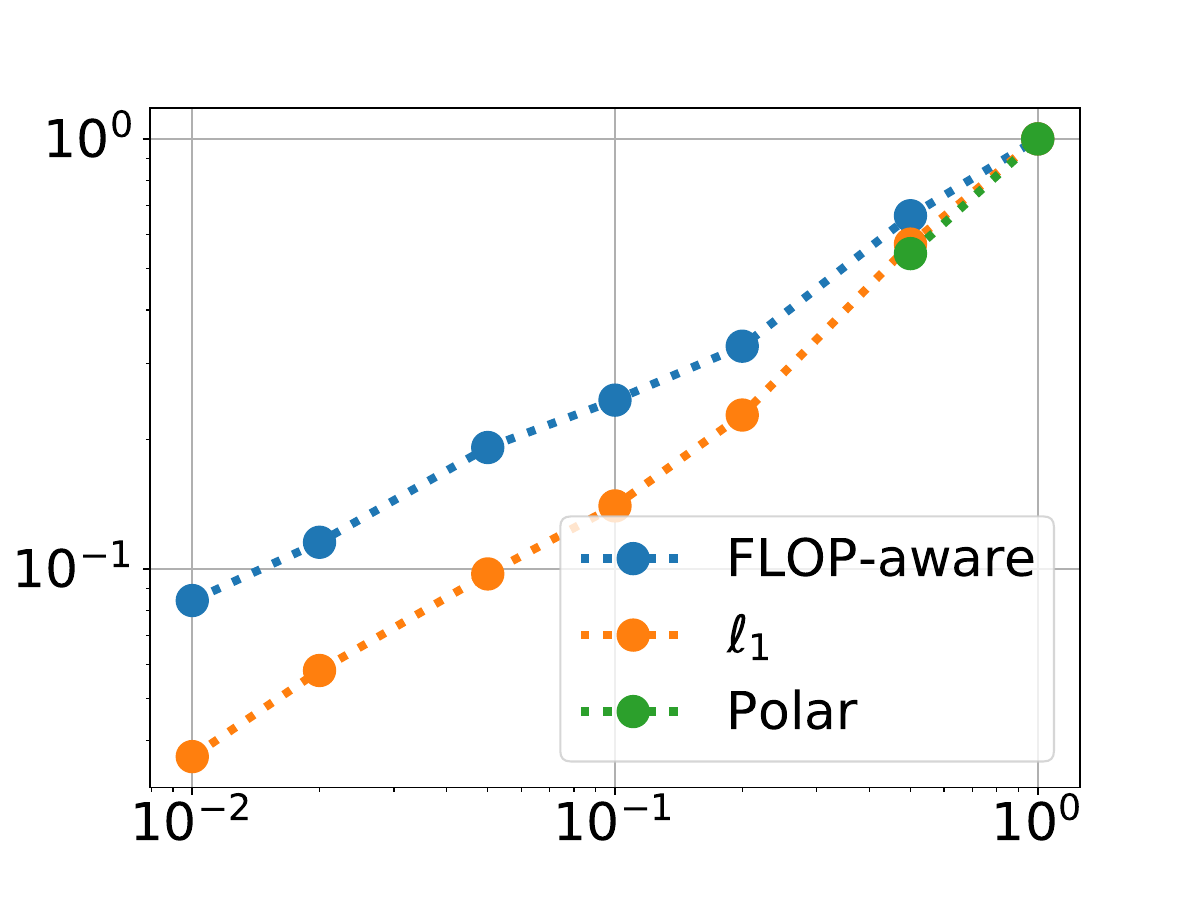}};
        \node at (0,-1.9) [scale=0.8] {FLOPs constraint $\gamma$};
        \node at (-2.6,0) [scale=0.6,rotate=90] {Fraction of Nonzero Channels};
    \end{tikzpicture}
    \caption{{Fraction of nonzero channels
    }}\label{fig:CIFAR10_ratio}
	
\end{subfigure}\hspace{0pt}\begin{subfigure}[t]{.33\textwidth}
  \centering
  \begin{tikzpicture}\hspace{-0pt}
        \node at (0,0) [scale=1.] {\includegraphics[width=\linewidth]{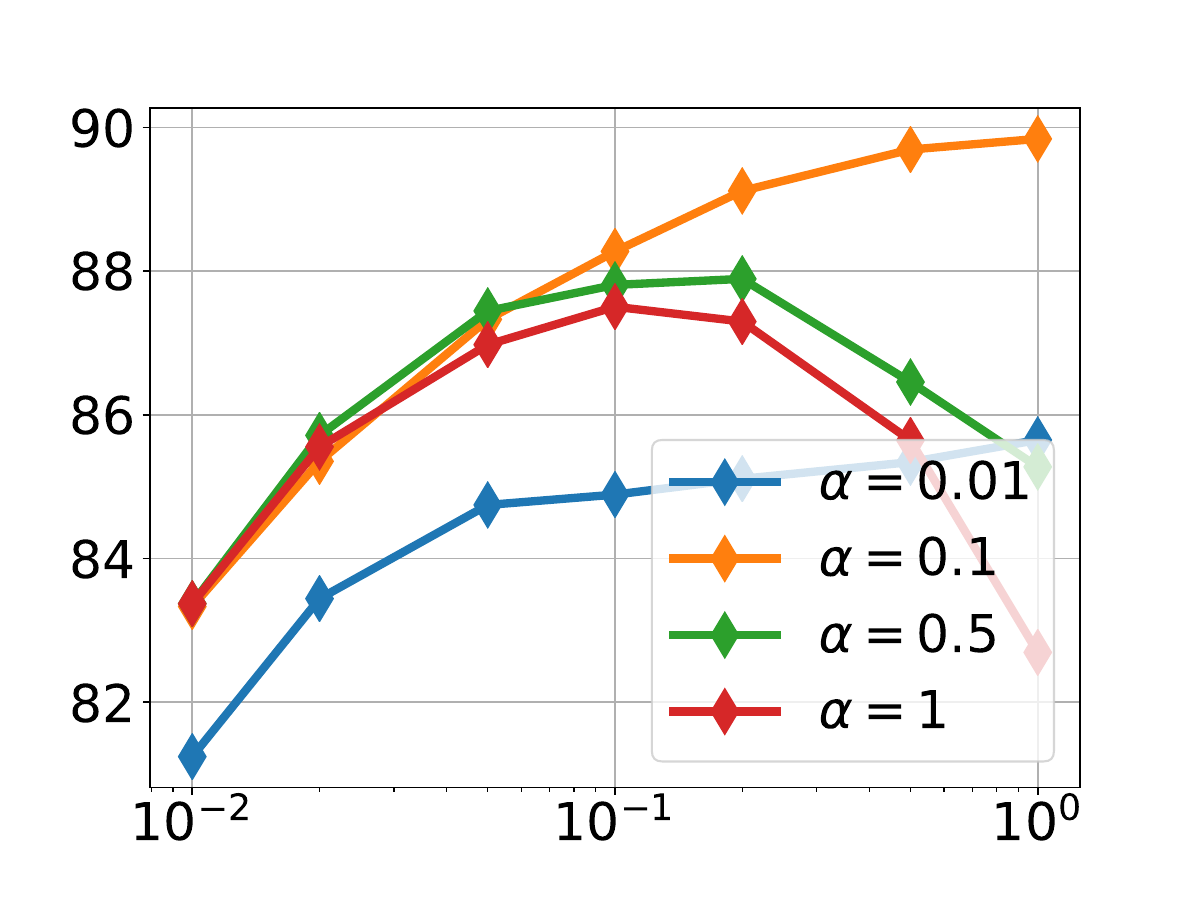}};
        \node at (0,-1.9) [scale=0.8] {FLOPs constraint $\gamma$};
        \node at (-2.4,0) [scale=0.8,rotate=90] {Accuracy};
    \end{tikzpicture}
    \centering
    \caption{{Impact of weight allocation $\alpha$
    }}\label{fig:FLOP}
	
\end{subfigure}\vspace{-10pt}
\caption{{
{Experimental evaluations of our FLOP-aware channel pruning and weight allocation techniques.} Fig.~\ref{fig:CIFAR10_acc} \& \ref{fig:CIFAR10_ratio} show channel pruning results for different methods. We compare our FLOP-aware pruning algorithm to the $\ell_1$-based~(\cite{liu2017learning}) and Polarization-based~(\cite{zhuang2020neuron}) methods over ResNet18 model used in our CL experiments. These experiments are conducted for standard classification tasks on CIFAR10 dataset and only focus on pruning performance (rather than CL setting). Our FLOP-aware pruning is displayed as the Blue curve, and it broadly outperforms the alternative approaches ($\ell_1$-based and Polarization-based). The performance gap is most notable in the very-few FLOPs regime (which is of more interest for inference-time efficiency).
In Fig.~\ref{fig:FLOP}, we display the impact of weight allocation $\alpha$ (Eq.~\eqref{wa-eq}) and FLOPs constraint $\gamma$ on \emph{task-averaged ESPN accuracy} for SplitCIFAR100. Setting $\alpha$ too small (Blue curve) degrades performance as tasks do not get enough nonzeros to train on. Setting both $\alpha$ and $\gamma$ to be large also degrades performance (Red \& Green curves), because the first few tasks get to occupy the whole supernetwork at the expense of future tasks. Finally $\alpha=0.1$ (Orange) achieves competitive performance for all $\gamma$ choices.
}}\vspace{-10pt}
\end{figure*}
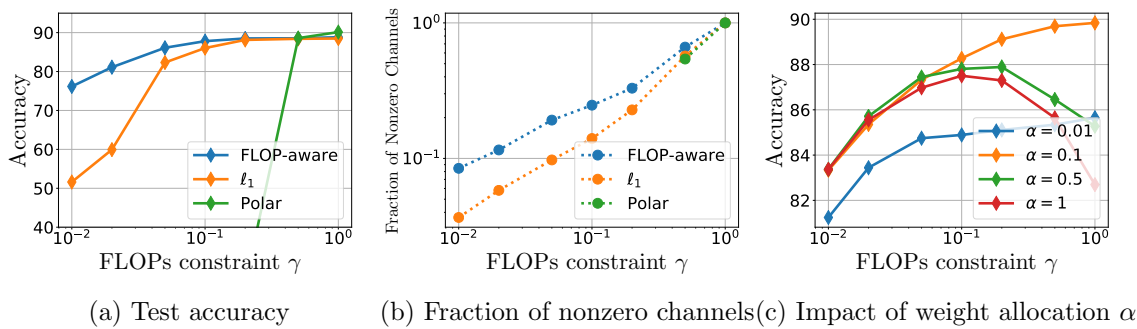
\noi\textbf{Empirical Benefits of Weight Allocation $\alpha$.} Figure \ref{fig:FLOP} displays the interaction between the FLOPs restriction $\gamma$ and weight allocation $\alpha$ in ESPN. Left handside shows that ESPN achieves respectable performance even using only $\gamma=1\%$ of MAX\_FLOPs. In our experiments in Table \ref{CIFARtable}, we use weight allocation $\alpha=0.1$ which is shown as the Orange curve. $\alpha=0.1$ strikes a balance between allocating too few (Blue curve) and too many (Red curve) new weights. As a result, accuracy gracefully improves as we relax the FLOPs constraint. {It also has competitive accuracy for all choices of $\gamma$. While $\alpha=0.1$ works well for SplitCIFAR100, we emphasize that optimal $\alpha$ depends on size of the supernet, number of tasks and difficulty of tasks. For instance, in our MNIST experiments with 36 tasks (Figure \ref{fig:mnist}), we chose $\alpha=0.05$. Parameter counting based on Eq.~\eqref{wa-eq} reveals that setting $\alpha$ inversely proportional to the total \# of tasks would enable fair allocation and good utilization of the network weights.} When we set $\alpha=1$, observe that, ESPN only uses channel pruning without any weight pruning. As a result, the comparison between Red ($\alpha=1$) and Orange ($\alpha=0.1$) curves demonstrates the accuracy benefits of our joint Channel \& Weight pruning strategy (over simply using channel pruning). 
%
\section{Proof of Theorem \ref{cl thm} and Corollary \ref{cl thm3}}\label{app B}
\red{In this section, we will first prove Theorem \ref{cl thm} which adds $T$ new tasks by leveraging a frozen feature extractor $\pf$. Then, we will prove Corollary~\ref{cl thm3} via directly applying Theorem~\ref{cl thm} to the neural net setting. } 

\subsection{Proof of Theorem \ref{cl thm}}
We begin with the statement of Bernstein's inequality showing in Fact~\ref{berns}, which will be used in the proof of Theorem~\ref{cl thm}.
\begin{fact}[Bernstein's inequality] \label{berns}Let $(X_i)_{i=1}^n$ be independent zero-mean random variables with $\red{\frac{1}{n}\sum_{i=1}^n} \E[X_i^2]\leq \sigma^2$ and $|X_i|\leq M$ almost surely. Then, 
\begin{align*}
    \Pro\left(\left|\frac{1}{n}\sum_{i=1}^n X_i\right|\geq \sigma\sqrt{\frac{2t}{n}} +\frac{Mt}{n}\right)\leq 2e^{-t}.
\end{align*}
\end{fact}
    \begin{proof} The standard Bernstein's inequality states $\Pro(|\frac{1}{n}\sum_{i=1}^n X_i|\geq \tau)\leq 2\exp(-\frac{\tau^2n/2}{\sigma^2+M\tau/3})$. To obtain above, we use the change of variable which implies
    \[
    t=\frac{\tau^2n/2}{\sigma^2+M\tau/3}\implies \tau^2n/2=\sigma^2t+Mt\tau/3\implies \tau\leq \sigma\sqrt{\frac{2t}{n}} +\frac{Mt}{n}.
    \]
    \end{proof}

\red{To proceed, we aim to prove Theorem~\ref{cl thm}.} The original statement of Theorem \ref{cl thm} does not introduce certain terms formally to provide a cleaner statement. Here, we first add some clarifying remarks on this. Let $\Bc(S)$ return the smallest Euclidean ball containing the set $S$ (that lies on an Euclidean space). For a function $f:\Zc\rightarrow\Zc'$, define the $\lin{\cdot}$ norm to be the Lipschitz constant $\lin{f}:=\lin{f}^{\Zc}=\sup_{\x,\y\in \Bc(\Zc),\x\neq\y}\frac{\tn{f(\x)-f(\y)}}{\tn{\x-\y}}$. Below, we assume that $\pf$ and all $\pn\in\bPn$, $h\in\Hb$ are $L$-Lipschitz (on their sets of input features). Suppose $\Xc$ is the set of feasible input features and has bounded Euclidean radius $R=\sup_{\x\in\Xc}\tn{\x}$. This means that input features of the classifier $h$ lie on the set $\Xc'=\pf\circ\Xc+\bPn\circ\Xc$ with Euclidean radius bounded by $2LR$. Thus, both raw features and intermediate features (output of $\phi=\pn+\pf$) are bounded {and we use these sets in Def.~\ref{def:cov}.} {Set $\BC:=\max(\BC_{\Xc},\BC_{\Xc'})$ below.}

\begin{theorem}[Theorem \ref{cl thm} restated] \label{cl thm2} Recall that $(\hhb,\hat{\phi}=\pnh+\pfh)$ is the solution of \eqref{crl}. Suppose that the loss function $\ell(y,\hat{y})$ takes values on $[0,\UB]$ and is $\Gamma$-Lipschitz w.r.t.~$\hat{y}$. Draw $T$ independent datasets $\{(\x_{ti},y_{ti})\}_{i=1}^N\subset \Xc\times \Yc$ for $t\in [T]$ where each dataset is distributed i.i.d.~according to $\Dc_t$. Suppose the input set $\Xc$ has bounded Euclidean radius. Suppose $\lin{\pf},\lin{\pn},\lin{h}\leq L$ for all $\pn\in\bPn,h\in\Hb$. With probability at least $1-2e^{-\tau}$, \red{for some absolute constant $C>4$}, the task-averaged population risk of the solution $(\hhb,\hat{\phi})$ obeys 
\red{
\begin{align}
\Lc(\hhb,\hat{\phi})&\leq \Lco_{\pfh}+ C\sqrt{\frac{\Lco_{\pfh} B(D+\tau)}{TN}}+\frac{CB(D+\tau)}{TN},\nn\\
&\leq \Lco+\MM+C\sqrt{\frac{(\Lco+\MM) B(D+\tau)}{TN}}+\frac{CB(D+\tau)}{TN}\nonumber,
\end{align}
where $D:=(T\cc{\Hb}+\cc{\bPn})\log(\BC\Gamma (L+1)NT)$.}
\end{theorem}
\begin{proof} Below $c,C>0$ denote absolute constants. For a scalar $a$, define $a_+=a+1$. The proof uses a covering argument following the definition of the covering dimension. Fix $1>\eps>0$. To start with, let $\bPne$ and $\Hb_\eps$ be $\eps$-covers (per Definition \ref{def:cov}) of the sets $\bPn$ and $\Hb$ respectively. Let $\Hb_\eps^T$ be $T$-times Cartesian product of $\Hb_\eps$. Our goal is bounding the supremum of the gap between the empirical and population risks to conclude the result. 


\noi\textbf{$\bullet$ Step 1: Union bound over the cover.} Following Definition \ref{def:cov}, we have that $\log|\Hb^T_\eps|\leq T\cc{\Hb}\log(\BC/\eps)$, $\log|\bPne| \leq \cc{\bPn}\log(\BC/\eps)$. Define $\Fc=\bPne\times \Hb^T_{\eps}$. These imply
\[
\log|\Fc|\leq D_\eps\quad\text{where}\quad  D_\eps:=(\cc{\bPn}+T\cc{\Hb})\log\left(\frac{\BC}{\eps}\right).
\]
Next, we will use two uniform concentration bounds to obtain an upper bound on the excess risk (based on $\bPne$ and $\Hb_\eps$). \red{Recall from \eqref{crl} that $\Lch(\hb,\phi)=\frac{1}{\Tn}\sum_{t=1}^\Tn\Lch_{\Sc_t}(h_t\circ\phi)$ where $\Lch_{\Sc_t}(h_t\circ\phi)=\frac{1}{N}\sum_{i=1}^N\ell(y_{ti},h_t\circ\phi(\x_{ti}))$. Let $\Lc_t(h_t\circ\phi)=\E[\Lch_{\Sc_t}(h_t\circ\phi)]$ and then $\Lc(\hb,\phi)=\frac{1}{\Tn}\sum_{t=1}^\Tn\Lc_t(h_t\circ\phi)$. Apply Fact~\ref{berns} with $X_{ti}=\ell(y_{ti},h_{t}\circ\phi(\x_{ti}))-\Lc_t(h_t\circ\phi)$ where $\{(\x_{ti},y_{ti})\}_{i=1}^{N}\sim\Dc_{t}$. Note that $\E[X_{ti}]=0$, $|X_{ti}|\leq B$, and
\[
\E[X_{ti}^2]\leq \E[\ell(y_{ti},h_{t}\circ\phi(\x_{ti}))^2]-\Lc_t(h_t\circ\phi)^2\leq B\cdot\Lc_t(h_t\circ\phi)-\Lc_t(h_t\circ\phi)^2\leq B\cdot\Lc_t(h_t\circ\phi).
\]
We can rewrite $\Lch(\hb,\phi)-\Lc(\hb,\phi)=\frac{1}{\Tn}\sum_{t=1}^\Tn(\Lch_{\Sc_t}(h_t\circ\phi)-\Lc_{t}(h_t\circ\phi))=\frac{1}{N\Tn}\sum_{t=1}^{\Tn}\sum_{i=1}^NX_{ti}$, where $\frac{1}{N\Tn}\sum_{t=1}^\Tn\sum_{i=1}^N\E[X_{ti}^2]\leq B\cdot\frac{1}{\Tn}\sum_{t=1}^T\Lc_t(h_t\circ\phi)=B\cdot\Lc(\hb,\phi)$. Then for each $(\hb,\phi)\in\Fc$, applying Fact \ref{berns} with $\sigma^2= B\cdot\Lc(\hb,\phi)$, we obtain that with probability at least $1-2e^{-\tau}$
\[
|\Lch(\hb,\phi)-\Lc(\hb,\phi)|\leq \sqrt{\frac{2\Lc(\hb,\phi)B\tau}{N\Tn}} +\frac{B\tau}{N\Tn}.
\]
Applying a union bound, with same probability, we obtain
\begin{align}
\sup_{(\hb,\phi)\in\Fc}|\Lch(\hb,\phi)-\Lc(\hb,\phi)|\leq \sqrt{\frac{2\Lc(\hb,\phi)B(D_\eps+\tau)}{N\Tn}} +\frac{B(D_\eps+\tau)}{N\Tn}\label{prob bound}.
\end{align}}

\noi\textbf{$\bullet$ Step 2: Perturbation analysis.} We showed the concentration on the covers. Now we need to relate the covers to the continuous sets. Given any candidate $\phi=\pn+\pf$ with $\pn\in\bPn,\bhb:=(h_t)_{t=1}^T\in\Hb^T$, we draw a neighbor $\phi'=\pn'+\pf$ with $\pn'\in\bPne,\bhb':=(h'_t)_{t=1}^T\in\Hb_\eps^T$.

Recall that $\lin{\pn},\lin{h}\leq L$ for all $\pn\in\bPn,h\in\Hb$. The task-averaged risk perturbation relates to the individual risks which in turn relates to individual examples as follows. Let $f_t:=h_t\circ \phi$ and $f'_t:=h'_t\circ \phi'$ for all $t\in[T]$.
\begin{align} 
|\Lc(\bhb,\phi)-\Lc(\bhb',\phi')|&\leq \sup_{1\leq t\leq T} |\Lc_t(h_t,\phi)-\Lc_t(h_t',\phi')|\\
&\leq\sup_{t\in[T],\red{(\x,y)\in\Xc\times\Yc}}|\ell(y,f_t(\x))-\ell(y,f'_t(\x))|\\
&\leq\Gamma\sup_{t\in[T],\red{\x\in\Xc}}|f_t(\x)-f'_t(\x)|.\label{pert1 bound}
\end{align}
The perturbations for individual examples are bounded via triangle inequalities as follows. 
\begin{align}
|f_t(\x)-f'_t(\x)|\leq  &|h_t\circ\phi(\x)-h'_t\circ\phi'(\x)|\nonumber\\
\leq & |h_t\circ\phi(\x)-h_t\circ\phi'(\x)|+|h_t\circ\phi'(\x)-h'_t\circ\phi'(\x)|\nonumber\\
\leq& (L+1)\eps:=L_+\eps.\label{last line}
\end{align}
The last line follows from the triangle inequality via $\eps$-covering and $L$-Lipschitzness of $h$ and $\bhi$ as follows
\begin{itemize}
\item Since $\tn{\phi(\x)-\phi'(\x)}=\tn{\pn(\x)-\pn'(\x)}\leq \eps\implies |h_t\circ\phi(\x)-h_t\circ\phi'(\x)|\leq L\eps$.
\item Set $\vb=\phi'(\x)$. Since $\tn{h(\vb)-h'(\vb)}\leq \eps\implies |h'_t\circ\phi'(\x)-h_t\circ\phi'(\x)|\leq \eps$.
\end{itemize}
Combining these, following \eqref{pert1 bound}, and repeating the identical perturbation argument \eqref{last line} for the empirical risk $\Lch$, we obtain
\begin{align}
\max(|\Lc(\bhb,\phi)-\Lc(\bhb',\phi')|, |\Lch(\bhb,\phi)-\Lch(\bhb',\phi')|)\leq \Gamma L_+\eps.\label{pert bound}
\end{align}

\noi\textbf{$\bullet$ Step 3: Putting things together.} Combining \eqref{prob bound} and \eqref{pert bound}, we found that, with probability at least $1-2e^{-\tau}$, for all $\bhi=\pn+\pf$ with $\pn\in\bPn$ and all $\bhb\in\Hb^T$, we have that
\red{
\begin{align}
|\Lc(\bhb,\phi)-\Lch(\bhb,\phi)|\leq 2\Gamma L_+\eps+\sqrt{\frac{2(\Lc(\hb,\phi)+\Gamma L_+\eps)B(D_\eps+\tau)}{N\Tn}} +\frac{B(D_\eps+\tau)}{N\Tn}.
\end{align}
Setting \blue{$\eps=\frac{1}{\Gamma L_+NT}$}, for an updated constant $C>0$, we find
\begin{align}
|\Lc(\bhb,\phi)-\Lch(\bhb,\phi)|\leq \sqrt{\frac{2\Lc(\hb,\phi)B(D+\tau)}{N\Tn}} +\frac{CB(D+\tau)}{N\Tn},\label{unif conv}
\end{align}
where \blue{$D=(\cc{\bPn}+T\cc{\Hb})\log(\BC\Gamma L_+NT)$} following the above definition of $D_\eps$.
}

\red{
Note that, the uniform concentration above also implies the identical bound for the minimizer of the empirical risk. Let $(\hat\bhb,\hat\bhi)$ be the minimizer of the empirical risk. Specifically, let $(\bhb^{\st,\pf},\pn^{\st,\pf})$ be the optimal hypothesis in $(\Hb,\bPn)$ minimizing the population risk subject to using frozen feature extractor $\pf$, that is, $\Lc(\bhb^{\st,\pf},\pn^{\st,\pf}+\pf)=\Lc^\st_{\pf}$. Then, we note that
\[
\Lc(\hat\bhb,\hat\bhi)-\Lch(\hat\bhb,\hat\bhi)\leq \sqrt{\frac{2\Lc(\hat\hb,\hat\phi)B(D+\tau)}{N\Tn}} +\frac{CB(D+\tau)}{N\Tn},
\]
\[
\Lch(\bhb^{\st,\pf},\pn^{\st,\pf}+\pf)-\Lc^\st_{\pf}\leq\sqrt{\frac{2\Lc^\st_{\pf}\cdot B(D+\tau)}{N\Tn}} +\frac{CB(D+\tau)}{N\Tn}.
\]
Since $(\hat\hb,\hat\phi)$ minimizes the empirical risk, $\Lch(\hat\bhb,\hat\bhi)\leq\Lch(\bhb^{\st,\pf},\pn^{\st,\pf}+\pf)$. Then, we can obtain
\[
    \Lc(\hat\bhb,\hat\bhi)-\Lc^\st_{\pf}\leq \sqrt{\frac{2\Lc(\hat\hb,\hat\phi)B(D+\tau)}{N\Tn}}+ \sqrt{\frac{2\Lc^\st_{\pf}\cdot B(D+\tau)}{N\Tn}}+\frac{2CB(D+\tau)}{N\Tn}.
\]
This in turn implies, for an updated constant $C'=2C+0.5$
\begin{align*}
    \left(\sqrt{\Lc(\hat\bhb,\hat\bhi)}-\sqrt{\frac{B(D+\tau)}{2N\Tn}}\right)^2&\leq\Lc^\st_{\pf}+\sqrt{\frac{2\Lc^\st_{\pf}\cdot B(D+\tau)}{N\Tn}}+\frac{C'B(D+\tau)}{N\Tn}\\
    &\leq\left(\sqrt{\Lc^\st_{\pf}}+\sqrt{\frac{C'B(D+\tau)}{N\Tn}}\right)^2.
\end{align*}
Now, observe that if $\Lc(\hat\bhb,\hat\bhi)\leq\frac{B(D+\tau)}{2N\Tn}$, the statement of the theorem holds. Otherwise, $\sqrt{\Lc(\hat\bhb,\hat\bhi)}-\sqrt{\frac{B(D+\tau)}{2N\Tn}}\geq0$ and we can obtain $\sqrt{\Lc(\hat\bhb,\hat\bhi)}\leq\sqrt{\Lc^\st_{\pf}}+\left(\sqrt{C'}+1\right)\sqrt{\frac{B(D+\tau)}{N\Tn}}$. Taking the square of both sides concludes the proof of the main statement (first line). The second statement follows directly from the definition of mismatch, that is, using the fact that $\MM=\Lc^\st_{\pf}-\Lc^\st$.
}
\end{proof}
\vspace{-10pt}
\subsection{Proof of Corollary~\ref{cl thm3}}\label{app nn proof}

\begin{proof}
Within this setting, classifier heads correspond to $h(\ab):=h_{\vb}(\ab)=\sigma(\vb^\top\psi(\ab))$ and $\Hb=\{h_{\vb}\bgl \tn{\vb}\leq \bz\}$. Similarly, feature representations  correspond to $\phi(\x)=\W\x$, $\pf(\x)=\W_\frz\x$, $\pn(\x):=\pn^{\W_\new}(\x)=\W_\new\x$ where $\W=\W_\frz+\W_\new\in\R^{r\times d}$. The hypothesis set becomes $\bPn=\{\pn^{\W_\new}\bgl \W_\new\in\Wc\}$.

Here $\W_\new\in\Wc$ is the weights of the new feature representation to learn on top of $\W_\frz$. Importantly, $\W_\new$ only learns the last $r_\new$ rows since first $r_\frz$ rows are fixed by frozen feature extractor $\W_\frz$. For the proof, we simply need to plug in the proper quantities within Theorem \ref{cl thm}. First, observe that $\Lc^\st=\E[Z]^2$ since $Z$ is independent zero-mean noise thus for any predictor using $\hat{y}:=\hat{y}(\x)$ input features we have
\[
\E[\ell(y,\hat{y})]=\E[(y(\x)+Z-\hat{y}(\x))^2]\geq \E[Z^2],
\]
where $y(\x)=\sigma({\vb^\st_t}^\top \psi(\W^\st\x))$ is the noiseless label.

We next prove that $\Lc^\st_\frz=\Lc^\st=\E[Z^2]$. This simply follows from the fact that frozen representation $\W_\frz$ is perfectly compatible with ground-truth representation $\W^\st$. Specifically, observe that $\W^\st_\new:=\W^\st-\W_\frz$ lies within the hypothesis set $\Wc$ since by construction $\|\W^\st_\new\|\leq \|\W^\st\|\leq \bo$ and $\W^\st_\new$ is zero in the first $r_\frz$ rows. Similarly $(\vb^\st_t)_{t=1}^T$ obey the $\ell_2$ norm constraint $\tn{\vb^\st_t}\leq \bz$. Thus, $\W^\st_\new,\Vb^\st=(\vb^\st_t)_{t=1}^T$ are feasible solutions of the hypothesis space and since $\W^\st_\new+\W_\frz=\W^\st$, for this choice we have that $\hat{y}(\x)=y(\x)$ thus task-specific risks induced by $(\vb^\st_t,\W^\st)$ obey $\Lc_t(\vb^\st_t,\W^\st)=\E[Z^2]$ for all $t\in[T]$. Consequently, the task-averaged risk obeys $\Lc(\Vb^\st,\W^\st)=\E[Z^2]$ proving aforementioned claim.

The remaining task is bounding the covering dimensions of the hypothesis sets $\Hb$ and $\bPn$ and verifying Lipschitzness. The Lipschitzness of $h(\ab)=\sigma(\vb^\top\psi(\ab))\in\Hb$, $\pf(\x)=\W_\frz\x$, $\pn(\x)=\W_\new\x\in\bPn$ follows from the fact that all $\W_\new\in\Wc,\W^\st,\W_\frz$ have spectral norms bounded by $\bo$, and the fact that, the Lipschitz constant of $h$ (denoted by $\lip{\cdot}$) can be bounded as $\lip{h}\leq \lip{\sigma} \lip{\psi}\tn{\vb}\leq \bar{L}^2 \bz$ where $\bar{L}=\max(\lip{\psi},\lip{\sigma})$. Recall that dependence on the Lipschitz constant is logarithmic. 

{What remains is determining the covering dimensions of $\Hb,\bPn$ which simply follows from covering the parameter spaces of $\vb\in\Bc^r(\bz),\W_\new\in\Wc$. Here $\Bc^r(\bz)\subset\R^r$ is defined to be the Euclidean ball of radius $\bz$. Suppose $\Xc$ lies on an Euclidean ball of radius $R$ and let $\Fc:=(\{\pf\}+\bPn)\circ \Xc$ be the feature representations. Since $\W_\frz$ and $\W_\new$ have spectral norm at most $\bo$, $\Fc$ is subset of Euclidean ball of radius $2\bo R$.}

Fix an $\eps_0=\frac{\eps}{2\bo R\bar{L}^2}$ $\ell_2$-cover of $\Bc^r(\bz)$. This cover has cardinality at most $(\frac{6\bz\bo R\bar{L}^2}{\eps})^r$ and induces an $\eps$-cover of $\Hb$. To see this, given any $\fb\in\Fc$ and $\vb\in \Bc^r(\bz)$, there exists a cover element $\vb'$ with $\tn{\vb'-\vb}\leq \eps_0$, as a result $|h_{\vb'}(\fb)-h_{\vb}(\fb)|\leq \bar{L}^2\tn{\fb}\tn{\vb'-\vb}\leq \eps$. Consequently $\cc{\Hb}=r$. Similarly, since elements of $\Wc$ has $r_\new d$ nonzero parameters (and recalling that $\Wc$ is also subset of spectral norm ball of radius $\bo$) $\Wc$ admits a $\frac{\eps}{R}$ Frobenius cover of cardinality $(\frac{3R\bo\sqrt{r_\new}}{\eps})^{r_\new d}$. Consequently, for any $\phi_{\W}$ with $\W=\W_\new+\W_\frz$, there exists  a cover element $\phi_{\W'}$ with $\|\W'_{\new}-\W_\new\|\leq \tf{\W'_\new-\W_\new}\leq \frac{\eps}{R}$, such that for all $\x\in\Xc$, we have that $|\phi_{\W}(\x)-\phi_{\W'}(\x)|\leq \|\W'_{\new}-\W_\new\|\tn{\x}\leq \eps$. Consequently $\cc{\bPn}\leq r_\new d$. These bounds on covering dimensions conclude the proof after applying Theorem \ref{cl thm}.
\end{proof}



\section{Proof of Theorem \ref{seq thm}}\label{app:proof seq thm}
Observe that Theorem \ref{cl thm} adds the $T$ new tasks in a multitask fashion which models the setting where new tasks arrive in batches of size $T$. This still captures continual learning due to the use of the frozen feature-extractor that corresponds to the features built by earlier tasks. Also, setting $T=1$ in Theorem \ref{cl thm} corresponds to adding a single new task and updating the representation. Using this observation, we provide an additional result Theorem \ref{seq thm} where the tasks are added to the super-network sequentially. Theorem \ref{seq thm}, provided in Section \ref{app C}, arguably better captures the CL setting. In essence, it follows as an iterative applications of Theorem \ref{cl thm} after introducing proper definitions that capture the impact of imperfections due to finite sample learning (see Definition \ref{def pop seq} and Assumption \ref{fin comp}). \red{To start with, we introduce the corollary below, which is a variant of Theorem~\ref{cl thm}. Since in sequential setting, $\Lco_{\pf}$ is task and supernet dependent, the following corollary decouples the population risk for learning a new task ($\Lco_{\pf}$) from the excess risk ($\Lc(\hhb,\hat{\phi})-\Lco_{\pf}$) and sequential results can be easily obtained via repeatedly applying it. }


\begin{corollary}\label{cl thm variant} \red{Under identical setting to Theorem \ref{cl thm}, with probability at least $1-2e^{-\tau}$, the task-averaged population risk of the solution $(\hhb,\hat{\phi})$ obeys} 
    {
    \begin{align}
    \Lc(\hhb,\hat{\phi})&\leq \Lco_{\pf}+ \sqrt{\frac{\ordet{\Tn\cc{\Hb}+\cc{\bPn}+\tau}}{\Tn N}},\nn\\
    &\leq \Lco+{\MM}+{\sqrt{\frac{\ordet{\Tn\cc{\Hb}+\cc{\bPn}+\tau}}{\Tn N}}}.\nn
    \end{align}}
\end{corollary} 
\red{This corollary follows immediately from the fact that the loss in Theorem \ref{cl thm} is assumed to be bounded ($\Lco_{\pf}\leq1$).}
\subsection{Proof of Theorem \ref{seq thm}}
\begin{proof} Applying Lemma \ref{lem seq} for each task $1\leq t\leq T$ and union bounding, \blue{with probability $1-2Te^{-\tau}$}, for all $T$ applications of \eqref{crlseq}, we have that
\begin{align}
&\Lc(h^t,\pn^t+\pf^t)-\Lci{t}\leq \MS{t}+\underbrace{\ME{t}+\sqrt{\frac{\ordet{\cc{\Hb}+\Ccn+\tau}}{N}}}_{\text{excess empirical error}}\label{gen bound}\\
&\MA{t}\leq \sqrt{\frac{\ordet{\cc{\Hb}+\Ccn+\tau}}{N}}+\ME{t}.\nn
\end{align}
\blue{We simply need to control $\ME{t}$ at time $t$. Following Assumption \ref{fin comp}, observe that at $t=1$ we simply have $\ME{1}=0$. For $\ME{t}$ we have that
\begin{align}
&\ME{t}\leq \cz+\frac{1}{t-1} \sum_{\tau=1}^{t-1} \MA{\tau}\nn\\
&\implies \MA{t}\leq \left[\cz+\sqrt{\frac{\ordet{\cc{\Hb}+\Ccn+\tau}}{N}}\right]+\frac{\sum_{\tau=1}^{t-1}\MA{\tau}}{t-1}.\label{required eq}
\end{align}
Declare $B:=\cz+\sqrt{\frac{\ordet{\cc{\Hb}+\Ccn+\tau}}{N}}$. We will inductively prove that for all $t$
\begin{align}
\MA{t}\leq (t-1) B.\label{induct}
\end{align}
For $t=1$, obviously it works. Suppose claim holds until time $t-1$. Using \eqref{required eq} this implies
\[
\MA{t}\leq B+\frac{\sum_{\tau=1}^{t-1}\MA{\tau}}{t-1}\leq B+  \frac{\sum_{\tau=1}^{t-1}(\tau-1)}{t-1}B=B+\frac{t-2}{2}B=\frac{Bt}{2}\leq B(t-1).
\]
The last inequality holds since $t\geq 2$. Thus the claim holds for $t$ as well.
}

To proceed, using the fact that $\ME{t}\leq \MA{t}$ and using the upper bound \eqref{induct}, the excess error in \eqref{gen bound} is upper bounded by $tB=(t-1)B+B$ to obtain
\begin{align}
\Lc(h^t,\pn^t+\pf^t)-\Lci{t}\leq \MS{t}+t \left(\cz+\sqrt{\frac{\ordet{\cc{\Hb}+\Ccn+\tau}}{N}}\right).\label{bound fin gen}
\end{align}
 To conclude, we simply sum up \eqref{bound fin gen} for $1\leq t\leq T$ to obtain the advertised bound where the total excess finite-sample error grows as $\frac{T(T+1)}{2}B\leq T^2B$. This yields \eqref{gen bound 3}. \eqref{gen bound 4} is identical to \eqref{gen bound 3} via \eqref{mst}.
\end{proof}

\subsection{Proof of Lemma~\ref{lem seq}}
\begin{proof} 
As introduced in \eqref{ltseq}, the representation quality of the new task will be captured by
\[
\Lc^t_\seq=\min_{h\in\Hb,\pn\in\bPn^t}\Lc_t(h,\phi)~\text{s.t.}~\phi=\pf^t+\pn.
\]
Similarly, recall the empirical mismatch $\ME{t}=\Lc^t_\seq-\Lci{t}_\seq$. 
Based on this and recalling $\MS{t}$ definition \eqref{mst}, applying Corollary \ref{cl thm variant} with $T=1$, we obtain that, with probability $1-2^{-\tau}$ we have the following bounds (on the same event)
\begin{align}
\Lc(h^t,\pn^t+\pf^t)&\leq\Lc^t_\seq+\sqrt{\frac{\ordet{\cc{\Hb}+\cc{\bPn^t}+\tau}}{N}}\nn\\
\Lc(h^t,\pn^t+\pf^t)&\leq \Lci{t}+\MS{t}+\ME{t}+\sqrt{\frac{\ordet{\cc{\Hb}+\cc{\bPn^t}+\tau}}{N}}.\nn
\end{align}
The second equation establishes \eqref{eq mm}.
The remaining challenge is simply relating the population and empirical mismatches i.e.~controlling $\ME{t}$. Recall $h^t$ is the ERM solution of \eqref{crlseq} associated with $\pn^t$. \red{Using the uniform concentration event (implied within the application of \eqref{unif conv} but via bounded loss $\Lc_t(h,\pn+\pf^t)$)}, we have
\[
|\Lch_{\Sc_t}(h,\pn+\pf^t)-\Lc_t(h,\pn+\pf^t)|\leq \sqrt{\frac{\ordet{\cc{\Hb}+\cc{\bPn^t}+\tau}}{N}},
\]
Now, using the optimality of $h^t$ for $\pn^t$ and the fact that $(h^t,\pn^t)$ minimizes the empirical risk, observe that
\begin{align*}
\MA{t}&=\min_{h\in\Hb}\Lc(h,\phi^t)-\Lci{t}_\seq\\
&\leq [\min_{h\in\Hb}\Lc(h,\phi^t)-\Lc^t_\seq]+\ME{t}\\
&\leq [\Lc(h^t,\phi^t)-\Lc^t_\seq]+\ME{t}\\
&\leq \sqrt{\frac{\ordet{\cc{\Hb}+\cc{\bPn^t}+\tau}}{N}}+\ME{t},
\end{align*}
to conclude with the second line of \eqref{eq mm}.
\end{proof}
{
\begin{table*}[t]
    \vspace{-10pt}
    \caption{Statistics of 6 tasks used in Fig.~\ref{fig:diversity}.}\label{table:crl_dataset}
    \centering
    
    \setlength{\tabcolsep}{0.7mm}{
    \begin{tabular}{|l|c|c|c|}\hline
    &\# of train &\# of test &\# of classes\\\hline
    ImageNet & 1,281,167 & 50,000   & 1,000\\
    CUBS    & 5,994     & 5,794     & 200\\
    Standford Cars & 8,144 & 8,041& 196\\
    Flowers & 2,040 & 6,149 & 102   \\
    Wikiart & 42,129 & 10,628 & 195  \\
    Sketch & 16,000 & 4,000 & 250 \\ \hline
    \end{tabular}}\vspace{-7pt}
\end{table*}
\begin{table*}[t]
\caption{Test accuracies on 6 tasks  corresponds to Fig.~\ref{fig:diversity}.}\label{table:crl_exp}
\centering
\begin{tabular}{|l|c|c|c|}
\hline
&\multirow{2}{*}{Individual} &\multicolumn{2}{c|}{ESPN}\\ \cline{3-4}
&&no pretrain & ImageNet pretrain\\\hline
ImageNet       & N/A   & N/A   & 68.91 \\
CUBS           & 49.64 & 49.64 & \textbf{57.77} \\
Standford Cars & 79.58 & 80.86 & \textbf{84.26} \\
Flowers        & 66.06 & 85.28 & \textbf{88.11} \\
Wikiart        & 65.74 & 65.76 & \textbf{67.69} \\
Sketch         & 73.52 & 74.25  & \textbf{75.80} \\
\hline
\end{tabular}

\end{table*}
\section{Experimental Setting of Figure~\ref{fig:diversity}}\label{app:imagenet}
Following \cite{mallya2018packnet,hung2019compacting}, we conduct experiments in Fig.~\ref{fig:diversity} to study how task order and diversity benefit CL. We use $6$ image classification tasks, where ImageNet-1k~(\cite{krizhevsky2012imagenet}) is the first task, followed by CUBS~(\cite{wah2011caltech}), Stanford Cars~(\cite{krause20133d}), Flowers~(\cite{nilsback2008automated}), WikiArt~(\cite{saleh2015large}) and Sketch~(\cite{eitz2012humans}). Table~\ref{table:crl_dataset} provides the detail information of all datasets. In the experiment, we train a standard ResNet50 model on the last 5 tasks (CUBS to Sketch) to explore how the representation learned from ImageNet can benefit CRL. We follow the same learning hyperparameter setting as Table~\ref{CIFARtable}, where we use batch size of 128 and Adam optimizer with $(\beta_1,\beta_2)=(0,0.999)$. Also we train $60$ and $90$ epochs for pre-training and pruning with learning rate $0.01$, and then we fine-tune the remained free weights for $100$ epochs using cosine decay over learning rate starting from $0.01$. 
Specifically, we employ weight allocation $\alpha=0.1$ (Eq.~\eqref{wa-eq}) in both individual and ESPN to enable continue representation learning. Compared with Individual/ESPN without pretraining, the ESPN with ImageNet pretrained employs a sparse pretrained ImageNet model from \cite{Wortsman2019DiscoveringNW} with 20\% non-zero weights. Table~\ref{table:crl_exp} shows the test accuracies when learning the 
 last 5 tasks and performances get improved in all 5 tasks. 
}


\end{document}